\documentclass[10pt,journal,compsoc]{IEEEtran}
\usepackage[utf8]{inputenc} 
\usepackage[T1]{fontenc}    
\usepackage{hyperref}       
\usepackage{url}            
\usepackage{booktabs}
\usepackage{graphicx}
\usepackage{mathrsfs}
\usepackage{amssymb}
\usepackage{amsfonts}
\usepackage{fdsymbol}
\usepackage{wrapfig}
\usepackage{caption}
\usepackage{subfigure}
\usepackage{multirow}
\usepackage{amsmath,bm}
\usepackage{graphicx}
\usepackage{amsfonts}       
\usepackage{nicefrac}       
\usepackage{microtype}      
\usepackage{hyperref}
\usepackage{amsthm}
\usepackage{color}

\newtheorem{myTheo}{Theorem}
\newtheorem{myPro}{Problem}

\newtheorem{myLemma}{Lemma}
\usepackage[ruled,linesnumbered,lined]{algorithm2e}
\usepackage{url}
\newcommand{\paratitle}[1]{\vspace{1.5ex}\noindent\textbf{#1}}
\usepackage{amsmath}
\usepackage{amssymb}
\usepackage{bm}
\usepackage{courier}
\usepackage{mathtools}
\usepackage{graphicx}
\usepackage{float}
\usepackage{subfigure}
\usepackage{amsmath}  
\usepackage{booktabs}
\usepackage{multirow}
\usepackage{color}
\usepackage{amsfonts}
\usepackage{caption}
\usepackage{enumerate}
\usepackage{enumitem}
\usepackage{bm}
\usepackage{stfloats}
\usepackage{graphicx}
\ifCLASSOPTIONcompsoc
  \usepackage[nocompress]{cite}
\else
  \usepackage{cite}
\fi
\ifCLASSINFOpdf
\else
\fi
\begin{document}
%
\title{Generalizing Graph Neural Networks on Out-Of-Distribution Graphs}

%
%
%
%

\author{Shaohua Fan,
        Xiao Wang,~\IEEEmembership{Member,~IEEE,}
        \IEEEauthorblockN{Chuan Shi\thanks{\IEEEauthorrefmark{2} Corresponding author.}\IEEEauthorrefmark{2}},~\IEEEmembership{Senior Member,~IEEE}, Peng Cui,~\IEEEmembership{Senior Member,~IEEE,} and
        Bai Wang
        
\IEEEcompsocitemizethanks{
\IEEEcompsocthanksitem S. Fan is with the Department
of Computer Science, Beijing University of Posts and Telecommunications,
Beijing 100876, China, and the Department of Computer Science and Technology in
Tsinghua University, Beijing 100084, China.
E-mail: fanshaohua@bupt.cn
\IEEEcompsocthanksitem X. Wang is with the School of Software, Beihang University, Beijing, 100191, China.
Email: xiao\_wang@buaa.edu.cn

\IEEEcompsocthanksitem C. Shi, and B. Wang are with the Department
of Computer Science, Beijing University of Posts and Telecommunications,
Beijing 100876, China.
E-mail: \{shichuan,wangbai\}@bupt.edu.cn.
\IEEEcompsocthanksitem P. Cui is with the Department of Computer Science and Technology in
Tsinghua University, Beijing 100084, China.
E-mail: cuip@tsinghua.edu.cn.}
\thanks{Manuscript received Nov 17, 2021; revised May 1, 2023; revised Sep 15, 2023;}}

%
%

\markboth{Journal of \LaTeX\ Class Files,~Vol.~14, No.~8, August~2015}%
{Shell \MakeLowercase{\textit{et al.}}: Bare Demo of IEEEtran.cls for Computer Society Journals}
%



\IEEEtitleabstractindextext{%
\begin{abstract}
\par Graph Neural Networks (GNNs) are proposed without considering the agnostic distribution shifts between training graphs and testing graphs, inducing the degeneration of the generalization ability of GNNs in Out-Of-Distribution (OOD) settings. The fundamental reason for such degeneration is that most GNNs are developed based on the I.I.D hypothesis. In such a setting, GNNs tend to exploit subtle statistical correlations existing in the training set for predictions, even though it is a spurious correlation. This learning mechanism inherits from the common characteristics of machine learning approaches. However, such spurious correlations may change in the wild testing environments, leading to the failure of GNNs. Therefore, eliminating the impact of spurious correlations is crucial for stable GNN models. To this end, in this paper, we argue that the spurious correlation exists among subgraph-level units and analyze the degeneration of GNN in causal view. Based on the causal view analysis, we propose a general causal representation framework for stable GNN, called StableGNN. The main idea of this framework is to extract high-level representations from raw graph data first and resort to the distinguishing ability of causal inference to help the model get rid of spurious correlations. Particularly, to extract meaningful high-level representations, we exploit a differentiable graph pooling layer to extract subgraph-based representations by an end-to-end manner. Furthermore, inspired by the confounder balancing techniques from causal inference, based on the learned high-level representations, we propose a causal variable distinguishing regularizer to correct the biased training distribution by learning a set of sample weights. Hence, GNNs would concentrate more on the true connection between discriminative substructures and labels. Extensive experiments are conducted on both synthetic datasets with various distribution shift degrees and eight real-world OOD graph datasets. The results well verify that the proposed model StableGNN not only outperforms the state-of-the-arts but also provides a flexible framework to enhance existing GNNs. In addition, the interpretability experiments validate that StableGNN could leverage causal structures for predictions. The source code is available at https://github.com/googlebaba/StableGNN.
\end{abstract}

}

\maketitle

\IEEEdisplaynontitleabstractindextext

%
\IEEEpeerreviewmaketitle

\begin{figure}[!htbp]
	\centering
	\includegraphics[width=8cm]{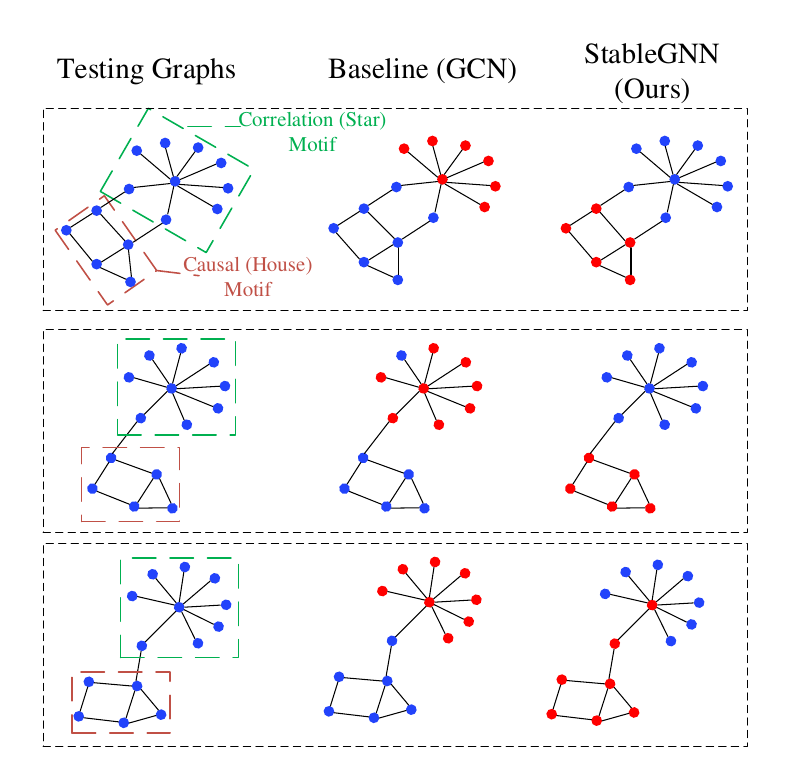}
	\caption{Visualization of subgraph importance for ``house'' motif classification task, produced by the vanilla GCN model and StableGNN when most of training graphs containing ``house'' motifs with ``star'' motifs. The red subgraph indicates the most important subgraph used by the model for prediction (generated by GNNExplainer~\cite{ying2019gnnexplainer}). Due to the spurious correlation, the GCN model tends to focus more on ``star'' motifs while our model focuses mostly on ``house'' motifs. For more cases of testing graphs, please refer to Figure~\ref{fig:syn_explainer}.}
	\label{fig:fig1}
\end{figure}

\IEEEraisesectionheading{\section{Introduction}\label{sec:introduction}}

\IEEEPARstart{G}raph Neural Networks (GNNs) are powerful deep learning algorithms on graphs with various applications~\cite{scarselli2008graph,kipf2016semi,velivckovic2017graph,hamilton2017inductive}. One major category of applications is the predictive task over entire graphs, i.e., graph-level tasks, such as molecular graph property prediction~\cite{hu2020open,lee2018graph,ying2018hierarchical}, scene graph classification~\cite{pope2019explainability}, and social network category classification~\cite{zhang2018end,ying2018hierarchical}, etc. The success could be attributed to the non-linear modeling capability of GNNs, which extracts useful information from raw graph data and encodes them into the graph representation in a data-driven fashion.

\par The basic learning diagram of existing GNNs is to learn the parameters of GNNs from the training graphs, and then make predictions on unseen testing graphs. The most fundamental assumption to guarantee the success of such a learning diagram is the I.I.D. hypothesis, i.e., training and testing graphs are independently sampled from the identical distribution~\cite{liao2020pac}. However, in reality, such a hypothesis is hardly satisfied due to the uncontrollable generation mechanism of real data, such as data selection biases, confounder factors or other peculiarities~\cite{bengio2019meta,engstrom2019exploring,su2019one,hendrycks2019benchmarking}. The testing distribution may incur uncontrolled and unknown shifts from the training distribution, called Out-Of-Distribution (OOD) shifts~\cite{sun2019test,krueger2021out}, which makes most GNN models fail to make stable predictions. As reported by OGB benchmark~\cite{hu2020open}, GNN methods will occur a degeneration of 5.66\% to 20\% points when splitting the datasets according to the OOD settings, i.e., splitting structurally different graphs into training and testing sets. 

\par Essentially, for general machine learning models, when there incurs a distribution shift, the main reason for the degeneration of accuracy is the spurious correlation between the irrelevant features and the category labels. This kind of spurious correlations is intrinsically caused by the unexpected correlations between irrelevant features and relevant features for a given category~\cite{zhang2021deep,lake2017building,lopez2017discovering}. For graph-level tasks that we focus on in this paper, as the predicted properties of graphs are usually determined by the subgraph units (e.g., groups of atoms and chemical bonds representing functional units in a molecule)~\cite{ying2018hierarchical,ying2019gnnexplainer,jin2020hierarchical}, we define one subgraph unit could be one relevant feature or irrelevant feature with graph label. Taking the classification task of ``house'' motif as an example, as depicted in Figure~\ref{fig:fig1}, where the graph is labeled by whether the graph has a "house" motif and the first column represents the testing graphs. The GCN model is trained on the dataset where ``house'' motifs coexist with ``star'' motifs in most training graphs. With this dataset, the structural features of ``house'' motifs and ``star'' motifs would be strongly correlated. This unexpected correlation leads to spurious correlations between structural features of ``star'' motifs with the label ``house''. And the second column of Figure~\ref{fig:fig1} shows the visualization  of  the most important subgraph used by the GCN for prediction (shown with red color and generated by GNNExplainer~\cite{ying2019gnnexplainer}). As a result, the GCN model tends to use such spurious correlation, i.e., "star" motif, for prediction. When encountering graphs without ``star'' motif, or other motifs (e.g., ``diamond'' motifs) with ``star'' motifs, the model is prone to make false predictions (See Section~\ref{sec::syn}).


\par To improve the OOD generalization ability of GNNs, one important way is to make GNNs get rid of such subgraph-level spurious correlation. However, it is not a trivial task, which will face the two following challenges: (1) How to explicitly encode the subgraph-level information into graph representation? As the spurious correlation usually exists between subgraphs, it is necessary to encode such subgraph information into the graph representation, so that we can develop a decorrelation method based on the subgraph-level representation. (2) How to remove the spurious correlation among the subgraph-level representations? The nodes in the same type of subgraph units may exist correlation, for example, the information of `N' atom and `O' atom in `$\text{NO}_2$' groups of molecular graphs may be encoded into different dimensions of learned embedding, but they act as an integrated whole and such correlations are stable across unseen testing distributions. Hence, we should not remove the correlation between the interior variables of one kind of subgraphs and only need to remove the spurious correlation between subgraph-level variables.

\par To address these two challenges, we propose a novel causal representation framework for graph, called \textbf{StableGNN}, which takes advantage of both the flexible representation ability of GNNs and the distinguishing ability for spurious correlations of causal inference methods. In terms of the first challenge, we propose a graph high-level variable learning module that employs a graph pooling layer to map nearby nodes to a set of clusters, where each cluster will be one densely-connected subgraph unit of original graph. Moreover, we theoretically prove that the semantic meanings of clusters would be aligned across graphs by an ordered concatenation operation. Given the aligned high-level variables, to overcome the second challenge, we analyze the degeneration of GNNs in causal view and propose a novel causal variable distinguishing regularizer to decorrelate each high-level variable pair by learning a set of sample weights. These two modules are jointly optimized in our framework. Furthermore, as shown in Figure~\ref{fig:fig1}, StableGNN can effectively partial out the irrelevant subgraphs (i.e., ``star'' motif) and leverage truly relevant subgraphs (i.e., ``house'' motif) for predictions. 
\par In summary, the major contributions of the paper are as follows:
\begin{itemize}
\item To our best knowledge, we are one of the pioneer works studying the OOD generalization problem on GNNs for graph-level tasks, which is a key direction to apply GNNs to wild non-stationary environments.
\item We propose a general causal representation learning framework for GNNs, jointly learning the high-level representation with the causal variable distinguisher, which could learn an invariant relationship between the graph causal variables with the labels. And our framework is general to be adopted for various base GNN models to help them get rid of spurious correlations.
\item Comprehensive experiments are conducted on both synthetic datasets and real-world OOD graph datasets. The effectiveness, flexibility and interpretability of the proposed framework have been well-validated with convincing results.
\end{itemize}

\section{Related Work}
In this section, we discuss three main categories closely related to our work: graph neural networks, causal representation learning and stable learning.

\subsection{Graph Neural Networks}
Graph neural networks are powerful deep neural networks that could perform on graph data directly~\cite{scarselli2008graph,kipf2016semi,velivckovic2017graph,hamilton2017inductive,zhang2018end,lee2018graph}. GNNs have been applied to a wide variety of tasks, including node classification~\cite{kipf2016semi,velivckovic2017graph,hamilton2017inductive,bo2021beyond,chen2023universal}, link prediction~\cite{schlichtkrull2018modeling,zhang2018link,fan2019metapath}, graph clustering~\cite{pan2023beyond,fan2020one2multi,wang2017community}, and graph classification~\cite{zhang2018end,lee2018graph}. For graph classification, the task we majorly focus on here, the major challenge in applying GNNs is how to generate the representation of the entire graph. Common approaches to this problem are simply summing up or averaging all the node embedding in the final layer. Several literatures argue that such simple operation will greatly ignore high-level structure that might be presented in the graph, and then propose to learn the hierarchical structure of graph in an end-to-end manner~\cite{zhang2018end,ying2018hierarchical,lee2019self,gao2019graph}.  Although these methods have achieved remarkable results in I.I.D setting, most of them largely ignore the generalization ability in OOD setting, which is crucial for real applications. During the review process, we notice several works claiming the importance of OOD generalization on graph classification tasks~\cite{li2022ood,wu2022discovering,chen2022learning,li2022learning}. Unlike the framework of OOD-GNN~\cite{li2022ood}, which learns a  single embedding for each graph and decorrelates each dimension of embeddings, our framework emphasizes the need to learn meaningful high-level representations for each subgraph and decorrelate those representations. DIR~\cite{wu2022discovering} divides a graph into causal- and non-causal part by an edge threshold. However, the threshold is set as the same for all graphs and is hard to select a good threshold for all graphs. CIGA~\cite{chen2022learning} proposes to maximize the agreement between the invariant part of graphs with the same labels. For example, the functional groups $\text{NO}_2$ and $\text{NH}_2$ could both determine the mutagenicity of a molecule. However, subgraphs with the same labels may not always be identical. DisC~\cite{fan2022debiasing} studies how to learn causal substructure in severe bias scenarios.  Other methods~\cite{li2022learning, yang2022learning} are based on the environmental inference framework, in which these methods iteratively infer the environment labels of graphs and learn the invariant information based on the environment labels.

\subsection{Causal Representation Learning}
Recently, Schölkopf et al.~\cite{scholkopf2021toward} publish a survey paper, towards causal representation learning, and point out that causality, with its focus on representing structural knowledge about the data-generating process that allows interventions and changes, can contribute towards understanding and resolving some limitations of current machine learning methods. Traditional causal discovery and reasoning assume that the units are random variables connected by a causal graph. However, real-world observations are not structured into these units, for example, graph data we focus on is a kind of unstructured data. Thus for this kind of data, causal representation learning generally consists of two steps: (1) inferring the abstract/high-level causal variables from available low-level input features, and (2) leveraging the causal knowledge as an inductive bias to learn the causal structure of high-level variables. After learning the causal system, the causal relationship is robust to irrelevant interventions and changes in data distributions. Based on this idea, Invariant Risk Minimization (IRM) framework~\cite{arjovsky2019invariant,rosenfeld2020risks,kamath2021does} proposes a regularization that enforces model learn an invariant representation across environments. The representation learning part plays the role of extracting high-level representation from low-level raw features, and the regularization encodes the causal knowledge that causal representation should be invariant across environments into the representation learning. Despite the IRM framework could learn the representations that have better OOD generalization ability,  the requirement that domains should be labeled hinders the IRM methods from real applications. To overcome such a dilemma, some methods are proposed to implicitly learn domains from data~\cite{qiao2020learning,matsuura2020domain,wang2019learning}, but they implicitly assume the training data is formed by balanced sampling from latent domains. However, all these methods are mainly designed for tabular or image data, thus they cannot capture the intrinsic properties of graphs.

\subsection{Stable Learning}
To improve the feasibility of IRM-based methods, a series of researches on stable learning are proposed~\cite{kuang2018stable,kuang2020stable}. These methods mainly bring the ideas from the causal effect estimation~\cite{angrist1995identification} into machine learning models. Particularly, Shen et al.~\cite{shen2018causally} propose a global confounding balancing regularizer that helps the logistic regression model to identify causal features, whose causal effect on outcomes are stable across domains. To make the confounder balancing much easier in high-dimensional scenarios, Kuang et al.~\cite{kuang2018stable} utilize the autoencoder to encode the high-dimensional features into low-dimensional representation. Furthermore, \cite{kuang2020stable} and \cite{shen2020sample} demonstrate that decorrelating relevant and irrelevant features can make a linear model produce stable predictions under distribution shifts. Nevertheless, they are all developed based on the linear framework. Recently, Zhang et al.~\cite{zhang2021deep} extend decorrelation into a deep CNN model to tackle more complicated data types like images. This method pays much attention to eliminating the non-linear dependence among features, but the feature learning part is largely ignored. We argue that it is important to develop an effective high-level representation learning method to extract variables with appropriate granularity for the targeted task, so that the causal variable distinguishing part could develop based on meaningful causal variables.

\begin{figure}[t]
	\centering
	\includegraphics[width=4cm]{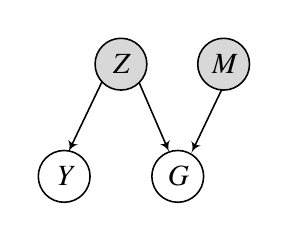}
	\caption{Causal graph for data generation process. Gray nodes and white nodes mean the unobserved latent variables and the observed variables, respectively.}
	\label{fig:fig2}
\end{figure}
\section{Problem Formulation}

\textbf{Notations.} In this paper, for any vector $\mathbf{v} = (\mathbf{v}_1, \mathbf{v}_2, \cdots, \mathbf{v}_n)$, let $\mathbf{v}_i$ the $i$-th element of $\mathbf{v}$. For any matrix $\mathbf{X}$, let $\mathbf{X}_{i,}$ and $\mathbf{X}_{,j}$ represent the $i$-th row and the $j$-th column in $\mathbf{X}$, respectively. $\mathbf{X}_{,j:k}$ denotes the submatrix of $\mathbf{X}$ from $j$-th column to $(k-1)$-th column. And for any italic uppercase letter, such as $X$, it will represent a random variable. We summarize the key notions used in the paper in Table~\ref{Tab:notations_2}.
\begin{table}[ht]
	\caption{Glossary of notations.}
	\centering
	\begin{tabular}{r|l}
		\hline
		\textbf{Notations} & \textbf{Description} \\ \hline
		$\mathcal{G}_{train}$/$\mathcal{G}_{test}$      & Training/Testing graphs         \\
		$\mathbf{A}$      &  Adjacency matrix         \\
		$\mathbf{F}$      & Feature matrix      \\
		$\mathbf{S}$      & Cluster assignment matrix      \\
		$\mathbf{Z}$      & Node representation matrix     \\
		$\mathbf{H}$      & High-level variable representation matrix     \\
		$\mathbf{P}$      & Permutation matrix      \\
		$T$      & Treatment variable           \\ 
		$Y$        & Label/prediction variable \\
		$X$ & Confounder variable          \\
		$X^{(p)}$ & The $p$-th high-level confounder variable          \\
		$\mathbf{X}^{(p)}$ & The $p$-th high-level confounder matrix          \\
		$\mathbf{w}$ & Sample/graph weights          \\
		\hline
	\end{tabular}
	\label{Tab:notations_2}
\end{table}
\label{Sec::Problem}
\begin{myPro}
\textbf{OOD Generalization Problem on Graphs.} Given the training graphs $\mathcal{G}_{train}=\{(G_1, Y_1), \cdots, (G_n, Y_n)\}$, where $G_i$ means the $i$-th graph data in the training set and $Y_i$ is the corresponding label, the task is to learn a GNN model $h_\theta(\cdot)$ with parameter $\theta$ to precisely predict the label of testing graphs $\mathcal{G}_{test}$, where the distribution $\Psi(\mathcal{G}_{train})\neq\Psi(\mathcal{G}_{test})$. And in the OOD setting, we do not know the distribution shifts from training graphs to unseen testing graphs.
\end{myPro}
We utilize a causal graph of the data generation process, as shown in Figure~\ref{fig:fig2}, to illustrate the fundamental reason to cause the distribution shifts on graphs. As illustrated in the figure, the observed graph data $G$ is generated by the unobserved latent cause $Z$ and $M$. $Z$ is a set of relevant variables and the label $Y$ is mainly determined by $Z$. $M$ is a set of irrelevant variables which does not decisive for label $Y$. During the unseen testing process in the real world, the variable $M$ could change, but the causal relationship P($Y$|$Z$) is invariant across environments.  Taking the classifying mutagenic property of a molecular graph~\cite{luo2020parameterized} as an example, $G$ is a molecular graph where the nodes are atoms and the edges are the chemical bonds between atoms, and $Y$ is the class label, e.g., whether the molecule is mutagenic or not. The whole molecular graph $G$ is an effect of relevant latent causes $Z$ such as the factors to generate nitrogen
dioxide ($\text{NO}_2$) group, which has a determinative effect on the mutagenicity of molecule, and the effect of irrelevant variable $M$, such as the carbon ring which exists more frequently in mutagenic molecules but not determinative~\cite{luo2020parameterized}. If we aim to learn a GNN model that is robust to unseen change on $M$, such as carbon exists in the non-mutagenic molecule, one possible way is to develop a representation function $f(\cdot)$ to recover $Z$ and $M$ from $G$, and learn a classifier $g(\cdot)$ based on $Z$, so that the invariant causal relationship $P(Y|Z)$ could be learned.

\section{The Proposed Method}
\subsection{Overview}
The basic idea of our framework is to design a causal representation learning method that could extract meaningful graph high-level variables and estimate their true causal effects for graph-level tasks. As depicted in Figure~\ref{fig:fig3}, the proposed framework mainly consists of two components: the graph high-level variable learning component and causal variable distinguishing component. The graph high-level variable learning component first employs a graph pooling layer that learns the node embedding and maps nearby nodes into a set of clusters. Then we get the cluster-level embeddings through aggregating the node embeddings in the same cluster, and align the cluster semantic space across graphs through an ordered concatenation operation. The cluster-level embeddings act as high-level variables for graphs. After obtaining the high-level variables, we develop a sample reweighting component based on Hilbert-Schmidt Independence Criterion (HSIC) measure to learn a set of sample weights that could remove the non-linear dependencies among multiple multi-dimensional embeddings. As the learned sample weights could generate a pseudo-distribution with less spurious correlation among cluster-level variables, we utilize the weights to reweight the GNN loss. Thus the GNN model trained on this less biased pseudo-data could estimate the causal effect of each high-level variable 
on the label more precisely, resulting in better generalization ability in wild environments.

\begin{figure*}[t]
	\centering
	\includegraphics[width=16cm]{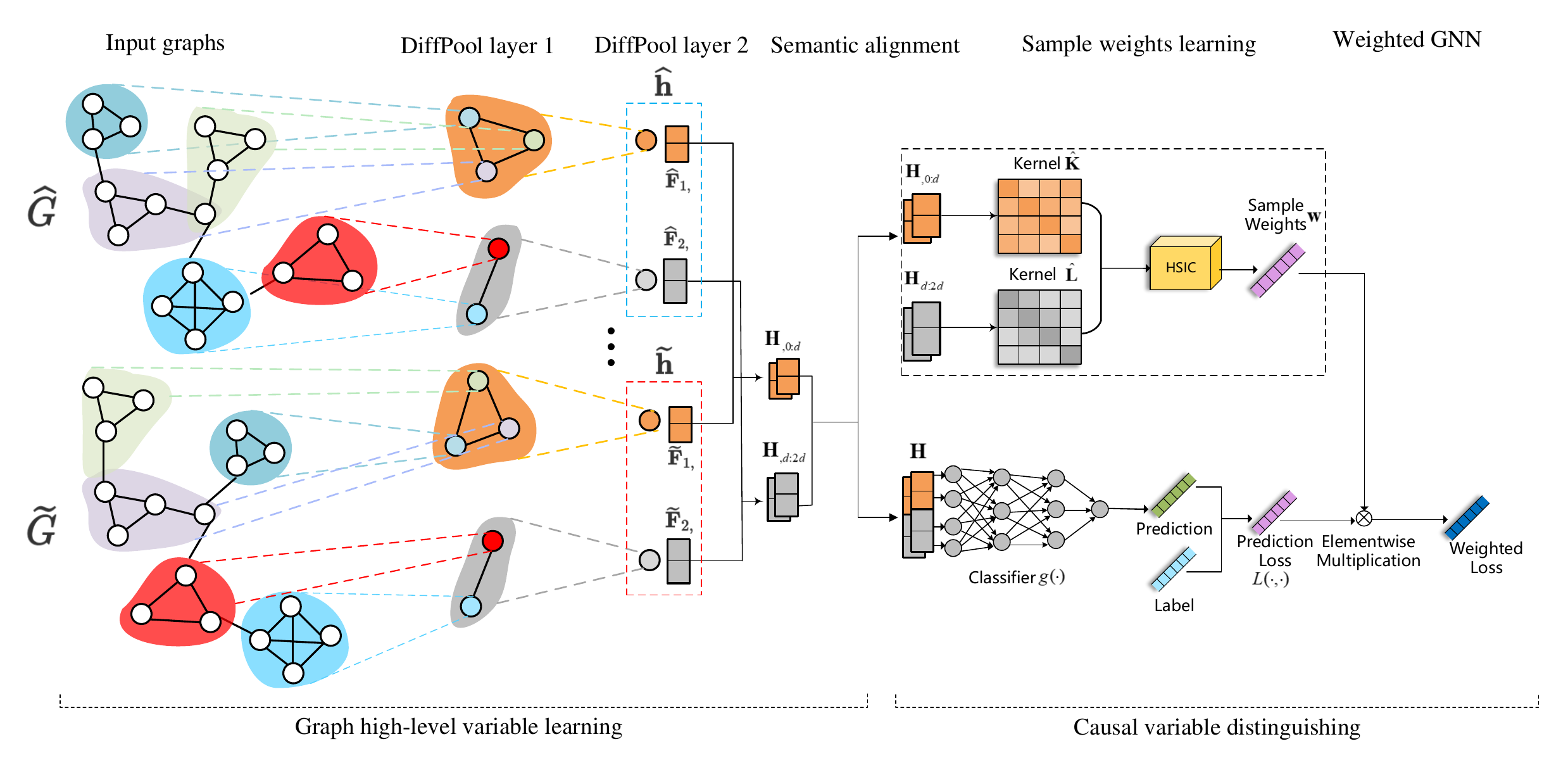}
	\caption{The overall architecture of the proposed StableGNN.}
	\label{fig:fig3}
\end{figure*}
\subsection{Graph High-level Variable Learning}
As the goal of this component is to learn a representation that could represent the original subgraph unit explicitly, we adopt a differentiable graph pooling layer that could map densely-connected subgraphs into clusters in an end-to-end manner. And then we theoretically prove that the semantic meaning of the learned high-level representations could be aligned by a simple ordered concatenation operation.

\paratitle{High-level Variable Pooling.} To learn node embedding as well as map densely connected subgraphs into clusters, we exploit the DiffPool layer~\cite{ying2018hierarchical} to achieve this goal by learning a cluster assignment matrix based on the learned embedding in each pooling layer. Particularly, as GNNs could smooth the node representations and make the representation more discriminative, given the input adjacency matrix $\mathbf{A}$, also denoted as $\mathbf{A}^{(0)}$ in this context, and the node features $\mathbf{F}$ of a graph, we firstly use an embedding GNN module to get the smoothed embeddings $\mathbf{F}^{(0)}$:
\begin{equation}
    \mathbf{F}^{(0)}=\text{GNN}^{(0)}_{\text{embed}}(\mathbf{A}^{(0)}, \mathbf{F}),
    \label{Eq:F0}
\end{equation}
where $\text{GNN}^{(0)}_{\text{embed}}(\cdot,\cdot)$ is a three-layer GNN module, and the GNN layer could be GCN~\cite{kipf2016semi}, GraphSAGE~\cite{hamilton2017inductive}, or others. Then we develop a pooling layer based on the smoothed representation $\mathbf{F}^{(0)}$. In particular, we first generate node representation at layer 1 as follows:
\begin{equation}
    \mathbf{Z}^{(1)}=\text{GNN}^{(1)}_{\text{embed}}(\mathbf{A}^{(0)}, \mathbf{F}^{(0)}).
    \label{Eq:Z}
\end{equation}
\par As nodes in the same subgraph would have similar node features and neighbor structure and GNN could map the nodes with similar features and structural information into similar representation, we also take the node embeddings $\mathbf{F}^{(0)}$ and adjacency matrix $\mathbf{A}^{(0)}$ into a pooling GNN module to generate a cluster assignment matrix at layer $1$, i.e., the clusters of original graph:
\begin{equation}
   \mathbf{S}^{(1)}=\text{softmax}(\text{GNN}^{(1)}_{\text{pool}}(\mathbf{A}^{(0)}, \mathbf{F}^{(0)})),
   \label{Eq:S}
\end{equation}
where $\mathbf{S}^{(1)}\in\mathbb{R}^{n_0\times n_{1}}$, and $n_0$ is the number of nodes of the input graph and $n_1$ is the number of the clusters at the layer 1, and $\mathbf{S}^{(1)}_{i,}$ represents the cluster assignment vector of $i$-th node, and $\mathbf{S}^{(1)}_{,j}$ corresponds to all nodes' assignment probabilities on $j$-th cluster at the layer $1$.  $\text{GNN}_{\text{pool}}^{(1)}(\cdot,\cdot)$ is also a three-layer GNN module, and the output dimension of $\text{GNN}_{\text{pool}}^{(1)}(\cdot,\cdot)$ is a pre-defined maximum number of clusters at layer $1$ and is a hyperparameter. And the appropriate number of clusters could be learned in an end-to-end manner. The maximum number of clusters in the last pooling layer should be the number of high-level variables we aim to extract. The softmax function is applied in a row-wise fashion to generate the cluster assignment probabilities for each node at layer $1$. 

\par After obtaining the assignment matrix $\mathbf{S}^{(1)}$, we could know the assignment probability of each node on the predefined clusters. Hence, based on the assignment matrix $\mathbf{S}^{(1)}$ and the learned node embedding matrix $\mathbf{Z}^{(1)}$, we could get a new coarsened graph, where the nodes are the clusters learned by this layer and edges are the connectivity strength between each pair of clusters. Particularly, the new matrix of embeddings is calculated by the following equation:
\begin{equation}
    \mathbf{F}^{(1)}={\mathbf{S}^{(1)}}^\mathrm{T}\mathbf{Z}^{(1)}\in\mathbb{R}^{n_{1}\times d},
    \label{Eq:F}
\end{equation}
where $d$ is the dimension of the embedding. This equation aggregates the node embedding $\mathbf{Z}^{(1)}$ according to the cluster assignment $\mathbf{S}^{(1)}$, generating embeddings for each of the $n_{1}$ clusters. Similarly, we generate a coarsened adjacency matrix as follows.
\begin{equation}
    \mathbf{A}^{(1)}={\mathbf{S}^{(1)}}^\mathrm{T}\mathbf{A}^{(0)}\mathbf{S}^{(1)}\in\mathbb{R}^{n_{1}\times n_{1}}.
    \label{Eq:A}
\end{equation}

For all the operations with superscript (1), it is one DiffPool unit and it is denoted as  $(\mathbf{A}^{(1)}, \mathbf{F}^{(1})=\text{DiffPool}(\mathbf{A}^{(0)}, \mathbf{F}^{(0)})$. For any DiffPool layer $l$, it could be denoted as $(\mathbf{A}^{(l)}, \mathbf{F}^{(l)})=\text{DiffPool}(\mathbf{A}^{(l-1)}, \mathbf{F}^{(l-1)})$. In particular, we could stack multiple DiffPool layers to extract the deep hierarchical structure of the graph.

\paratitle{High-level Representation Alignment.} After stacking $L$ graph pooling layers, we could get the most high-level cluster embedding $\mathbf{F}^{(L)}\in\mathbb{R}^{{n_L}\times d}$, where $\mathbf{F}^{(L)}_{i,}$ represents the $i$-th high-level representation of the corresponding subgraph in the original graph and $n_L$ is the number of high-level representation. As our target is to encode subgraph information into graph representation and $\mathbf{F}^{(L)}$ has explicitly encoded each densely-connected subgraph information into each row of $\mathbf{F}^{(L)}$, we propose to utilize the embedding matrix $\mathbf{F}^{(L)}$ to represent the high-level variables. However, due to the Non-Euclidean property of graph data, for the $i$-th learned high-level representations $\widehat{\mathbf{F}}^{(L)}_{i,}$ and $\widetilde{\mathbf{F}}^{(L)}_{i,}$ from two graphs $\widehat{G}$ and $\widetilde{G}$, respectively, their semantic meaning may not be matched, e.g., $\widehat{\mathbf{F}}^{(L)}_{i,}$ and $\widetilde{\mathbf{F}}^{(L)}_{i,}$ may represent scaffold substructure (e.g., carbon ring) and functional group (e.g., $\text{NO}_2$) in two molecular graphs, respectively. To match the semantic meaning of learned high-level representation across graphs, we propose to concatenate the high-level variables by the order of row index of high-level embedding matrix for each graph:
\begin{equation}
    \mathbf{h} = \text{concat}(\mathbf{F}^{(L)}_{1,},\mathbf{F}^{(L)}_{2,},\cdots,\mathbf{F}^{(L)}_{n_{L},}),
\end{equation}
where $\mathbf{h}\in\mathbb{R}^{n_L d}$ and $\text{concat}(\cdot)$ means concatenation operation by the row axis. Moreover, we stack $m$ high-level representations for a mini-batch with $m$ graphs to obtain the embedding matrix $\mathbf{H}\in\mathbb{R}^{m\times n_L d}$:
\begin{equation}
    \mathbf{H} = \text{stack}(\mathbf{h}_1,\mathbf{h}_{2},\cdots,\mathbf{h}_{m}),
\end{equation}
where $\mathbf{h}_i$ is the concatenated high-level representation of sample $i$, $\text{stack}(\cdot)$ means the stacking operation by the row axis, and $\mathbf{H}_{i,}$ means the high-level representation for the $i$-th graph in the mini-batch. $\mathbf{H}_{i,(k-1)d:kd}$ means $k$-th high-level representation of $i$-th graph. Hence, to prove the semantic alignment of any two high-level representations $\mathbf{H}_{i,}$ and $\mathbf{H}_{j,}$, we need to demonstrate that the semantic meanings of $\mathbf{H}_{i,(k-1)d:kd}$ and $\mathbf{H}_{j,(k-1)d:kd}$ are aligned for all $k \in [1,n_L]$. To this end, we first prove the permutation invariant property of DiffPool layer.
\begin{myLemma}
\textbf{Permutation Invariance~\cite{ying2018hierarchical}.} Given any permutation matrix $\mathbf{P}\in\{0,1\}^{n\times n}$, if $\mathbf{P}\cdot\text{GNN}(\mathbf{A},\mathbf{F})=\text{GNN}(\mathbf{P}\mathbf{A}\mathbf{P}^\mathrm{T}, \mathbf{P}\mathbf{F})$ (i.e., the GNN method used is permutation equivariant), then $\text{DiffPool}(\mathbf{A}, \mathbf{F})=\text{DiffPool}(\mathbf{P}\mathbf{A}\mathbf{P}^\mathrm{T}, \mathbf{P}\mathbf{F})$.
\end{myLemma}
\begin{proof}
Following~\cite{keriven2019universal}, a function $f:\mathbb{R}^{n\times d}\rightarrow\mathbb{R}^{n\times l}$: is  \textit{invariant} if $f(\textbf{F}, \textbf{X})=f(\textbf{PFP}^\mathrm{T}, \textbf{PF})$, i.e, the permutation will not change node representation and the node order in the learned matrix. A function $f:\mathbb{R}^{n\times d}\rightarrow\mathbb{R}^{n\times l}$: is \textit{equivariant} if $f(\textbf{A}, \textbf{F})=\textbf{P}\cdot f(\textbf{PAP}^\mathrm{T}, \textbf{PF})$, i.e., the permutation will not change the node representation but the order of nodes in matrix will be permuted. The permutation invariant aggregator functions such as mean pooling and max pooling ensure that the basic model can be trained and applied to arbitrarily ordered node neighborhood feature sets, i.e, permutation equivariant. Therefore, the Eq.~(\ref{Eq:Z}) and ~(\ref{Eq:S}) are the permutation equivariant by the assumption that the GNN module is permutation equivariant. And since any permutation matrix is orthogonal, i.e., $\mathbf{P}^\mathrm{T}\mathbf{P}=\mathbf{I}$, applying this into Eq.~(\ref{Eq:F}) and~(\ref{Eq:A}), we have:
\begin{equation}
    \mathbf{F}^{(1)}={\mathbf{S}^{(1)}}^\mathrm{T}\mathbf{P}^\mathrm{T}\mathbf{P}\mathbf{Z}^{(1)},
    \label{Eq:F1}
\end{equation}
\begin{equation}
    \mathbf{A}^{(1)}={\mathbf{S}^{(1)}}^\mathrm{T}\mathbf{P}^\mathrm{T}\mathbf{P}\mathbf{A}^{(1)}\mathbf{P}\mathbf{P}^\mathrm{T}\mathbf{S}^{(1)}.
    \label{Eq:A1}
\end{equation}
Hence, DiffPool layer is permutation invariant.
\end{proof}
After proving the permutation invariant property of a single graph, we then illustrate the theoretical results for the semantic alignment of learned high-level representation across graphs.
\begin{myTheo}
\textbf{Semantic Alignment of High-level Variables.} Given any two graphs $G_i=(\mathbf{A}_i, \mathbf{F}_i)$ and $G_j=(\mathbf{A}_j, \mathbf{F}_j)$, the semantic space of high-level representations $\mathbf{H}_{i,}$ and $\mathbf{H}_{j,}$ learned by a series of shared DiffPool layers is aligned.
\end{myTheo}
\begin{proof}
WL-test~\cite{leman1968reduction} is widely used in graph isomorphism checking: if two graphs are isomorphic, they will have the same multiset of WL colors at any iteration. GNNs could be viewed as a continuous approximation to the WL test, where they both aggregate signal from neighbors and the trainable neural network aggregators is an analogy to the hash function in WL-test~\cite{zhang2018end}. Therefore, the outputs of GNN module are exactly the continuous WL colors. The outputs of GNN layers or the WL color of nodes would represent their structural roles~\cite{zhang2018end}. For example, the output assignment matrix of the pooling GNN module in Diffpool layer $\mathbf{S}^{(1)}\in \mathbb{R}^{n_0\times n_1}$ could be interpreted as the structural roles of $n_0$ input nodes with respect to the $n_1$ clusters of layer 1. Through the assignment matrix $\mathbf{S}^{(1)}$, it will aggregate the nodes with similar structural roles into the clusters with the same index. Hence, the semantic meaning of each column of the learned assignment matrices $\mathbf{S}_i^{(l)}$ and $\mathbf{S}_j^{(l)}$ of graphs $G_i$ and $G_j$ would be aligned, e.g., the $k$-th column would represent carbon ring structure in all molecule graphs, which act as the scaffold structural role in molecule graphs. Due to the permutation invariant property of DiffPool layer, according to Eq. (\ref{Eq:F1}), the input graph with any node order will map its carbon ring structure signal into the $k$-th high-level representation, i.e., the $k$-th high-level representation of $\mathbf{H}_{i, (k-1)d:kd}$ and $\mathbf{H}_{j, (k-1)d:kd}$. Hence, the semantic meaning of $\mathbf{H}_{i,.}$ and $\mathbf{H}_{j,.}$ is aligned.
\end{proof} 

\subsection{Causal Variable Distinguishing Regularizer}
So far the variable learning part extracts all the high-level representations for all the densely-connected subgraphs regardless of whether it is relevant with the graph label due to causal relation or spurious correlation. In this section, we first analyze the reason leading to the degeneration of GNNs on OOD scenarios in a causal view and then propose the Causal Variable Distinguishing (CVD) regularizer with sample reweighting technique.

\begin{figure}[htbp]
\centering
\subfigure[Confounder balancing framework.]{
\includegraphics[width=4cm]{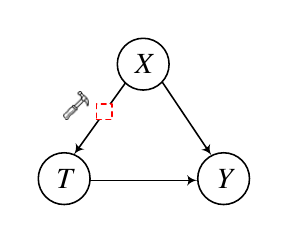}
\label{fig:OurCBa}
}
\subfigure[Multiple multi-dimensional confounder balancing framework.]{
\includegraphics[width=4cm]{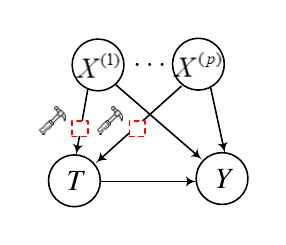}
\label{fig:OurCBb}
}
\caption{Causal view on GNNs.}
\label{fig:causal view}
\end{figure}

\paratitle{Revisiting on GNNs in Causal View.} As described in Section~\ref{Sec::Problem}, our target is to learn a classifier $g(\cdot)$ based on the relevant variable $Z$. To this end, we need to distinguish which variable of learned high-level representation $\mathbf{H}$ belongs to $Z$ or $M$. The major difference between $Z$ and $M$ is whether they have a causal effect on $Y$. For a graph-level classification task, after learning the graph representation, it will be fed into a classifier to predict its label. The prediction process could be represented by the causal graph with three variables and their relationships, as shown in Figure~\ref{fig:OurCBa}, where $T$ is a treatment variable, $Y$ is the prediction/outcome variable\footnote{We use variable $Y$ for both the ground-truth labels and prediction, as they are optimized to be the same.}, and $X$ is the confounder variable, which has effects both on the treatment variable and the outcome variable. The path $T\rightarrow Y$ represents the target of GNNs that aims to estimate the causal effect of the one learned variable $T$ (e.g., $i$-th high-level representation $\mathbf{H}_{, (i-1)d:id}$ ) on $Y$. Meanwhile, other learned variables (e.g., $j$-th high-level representation $\mathbf{H}_{, (j-1)d:jd}$) will act as confounder $X$. Due to the existence of spurious correlations of subgraphs, there are spurious correlations between their learned embeddings. Hence, there incurs a path between $X$ and $T$. \footnote{Note that the direction of arrow means that the assignment of treatment value will dependent on the confounder. However, if the arrow is reversed, it will not affect the theoretical results in our paper.}  And because GNNs also employ the confounder (e.g., the representation of carbon ring) for prediction, there exists a path $X\rightarrow Y$. Hence, these two paths form a backdoor path between $X$ and $Y$ (i.e., $T\leftarrow X\rightarrow Y$) and induce the spurious correlation between $T$ and $Y$. And the spurious correlation will amplify the true correlation between treatment variable and label, and may change in testing environments. Under this scenario, existing GNNs cannot estimate the causal effect of subgraphs accurately, so the performance of GNNs will degenerate when the spurious correlation change in the testing phase. Confounding balancing techniques~\cite{hainmueller2012entropy, zubizarreta2015stable} correct the non-random assignment of treatment variable by balancing the distributions of confounder across different treatment levels. Because moments can uniquely determine a distribution, confounder balancing methods directly balance the confounder moments by adjusting weights of samples~\cite{athey2016approximate,hainmueller2012entropy}. The sample weights $\mathbf{w}$ for binary treatment scenario are learnt by:
\begin{equation}
    \mathbf{w}=\mathop{\arg\min}_{\mathbf{w}}||\frac{\sum_{i:T_i=1} \mathbf{w}_i\cdot \mathbf{X}_{i,}}{\sum_{i:T_i=1}\mathbf{w}_i}-\frac{\sum_{j:T_j=0} \mathbf{w}_j\cdot \mathbf{X}_{j,}}{\sum_{j:T_j=0} \mathbf{w}_j}||_2^2,
\end{equation}
where $\mathbf{X}$ is the confounder matrix and $\mathbf{X}_{i,}$ is the confounder vector for $i$-th graph. Given a binary treatment feature $T$, $\frac{\sum_{i:T_i=1} \mathbf{w}_i\cdot \mathbf{X}_{i,}}{\sum_{i:T_i=1}\mathbf{w}_i}$ and $\frac{\sum_{j:T_j=0} \mathbf{w}_j\cdot \mathbf{X}_{j,}}{\sum_{j:T_j=0} \mathbf{w}_j}$ refer to the mean value of confounders on samples with and without treatment,
respectively. After confounder balancing, the dependence between $X$ and $T$ (i.e., $T\leftarrow X$) would be eliminated, thus the correlation between the treatment variable and the output variable will represent the causal effect (i.e., $X\rightarrow Y$).

\par Moreover, for a GNN model, we have little prior knowledge on causal relationships between the learned high-level variables $\{\mathbf{H}_{,0:d},\cdots, \mathbf{H}_{,(n_L-1)d:n_Ld}\}$, thus we have to set each learned high-level variable as treatment variable one by one, and the remaining high-level variables are viewed as confounding variables, e.g, $\mathbf{H}_{,0:d}$ is set as treatment variable and $\{\mathbf{H}_{,d:2d},\cdots, \mathbf{H}_{,(n_L-1)d:n_Ld}\}$ are set as confounders.  Note that for a particular treatment variable, previous confounder balancing techniques are mainly designed for a single-dimensional treatment feature as well as the confounder usually consists of multiple variables where each variable is a single-dimensional feature. In our scenario, however, as depicted in Figure~\ref{fig:OurCBb}, we should deal with the confounders which are composed of multiple multi-dimensional confounder variables $\{X^{(1)}, \cdots, X^{(p)}\}$, where each multi-dimensional confounder variable corresponds to one of learned high-level representations  $\{\mathbf{H}_{,0:d},\cdots, \mathbf{H}_{,(n_L-1)d:n_Ld}\}$. The multi-dimensional variable unit usually has integrated meaning such as representing one subgraph, so it is unnecessary to remove the correlation between the treatment variable with each of feature $\mathbf{H}_{,i}$ in one multi-dimensional feature unit, e.g., $\mathbf{H}_{,0:d}$. And we should only remove the subgraph-level correlation between treatment variable $T$ with multiple multi-dimensional variable units $\{X^{(1)}, \cdots, X^{(p)}\}$. One possible way to achieve this goal is to learn a set of sample weights that balance the distributions of all confounding variables for the targeted treatment variable, as illustrated in Figure~\ref{fig:OurCBb}, i.e., randomizing the assignment of treatment $T$ \footnote{Here, we still assume the treatment is a binary variable. In the following part, we will consider the treatment as a multi-dimensional variable.}  with confounders $\{X^{(1)}, \cdots, X^{(p)}\}$. The sample weights $\mathbf{w}$ could be learnt by the following \textit{multiple multi-dimensional confounder balancing} objective:
\begin{equation}
\mathbf{w}=\mathop{\arg\min}_{\mathbf{w}}\sum_{k=1}^p||\frac{\sum_{i:T_i=1} \mathbf{w}_i\cdot \mathbf{X}_{i,}^{(k)}}{\sum_{i:T_i=1} \mathbf{w}_i}-\frac{{\sum_{j:T_j=0} \mathbf{w}_j\cdot \mathbf{X}_{j,}^{(k)}}}{\sum_{j:T_j=0} \mathbf{w}_j}||_2^2,
\end{equation}
where $\mathbf{X}^{(k)}$ is the embedding matrix for $k$-th confounding variable $X^{(k)}$.

\paratitle{Weighted Hilbert-Schmidt Independence Criterion.} Since the above confounder balancing method is mainly designed for binary treatment variable, which needs the treatment value to divide the samples into treated or control group, it is hard to apply to the continuous multi-dimensional variables learned by GNNs. Inspired by the intuition of confounder balancing which is to remove the dependence between the treatment with the corresponding confounding variables, we propose to remove dependence between continuous multi-dimensional random variable $T$ with each of confounder in $\{X^{(1)}, \cdots, X^{(p)}\}$. Moreover, as the relationship between representations learned by the representation module is highly non-linear, it is necessary to measure the nonlinear dependence between them. And it is feasible to resort to HSIC measure~\cite{song2007supervised}. For two random variables $U$ and $V$ and kernel $k$ and $l$, HSIC is defined as $\text{HSIC}^{k,l}(U,V):=||C^{k,l}_{UV}||^2_\text{HS}$, where $C^{k,l}$ is a cross-covariance operator in the Reproducing Kernel Hilbert Spaces (RKHS) of $k$ and $l$~\cite{gretton2005measuring}, an RKHS analogue of covariance matrices. $||\cdot||_\text{HS}$ is the Hilbert-Schmidt norm, a Hilbert-space analogue of the Frobenius norm. For two random variables $U$ and $V$ and radial basis function (RBF) kernels $k$ and $l$, $\text{HSIC}^{k,l}(U,V)=0$ if and only if $U\perp V$. To estimate $\text{HSIC}^{k,l}(U,V)$ with finite sample, we employ a widely used estimator $\text{HSIC}^{k,l}_0(U,V)$~\cite{gretton2005measuring}  with $m$ samples, defined as:
\begin{equation}
    \text{HSIC}^{k,l}_0(U,V)=(m-1)^{-2}tr(\mathbf{KPLP}),
    \label{Eq:HSIC_0}
\end{equation}
where $\mathbf{K},\mathbf{L}\in\mathbb{R}^{m\times m}$ are RBF kernel matrices containing entities $\mathbf{K}_{ij}=k(U_{i}, U_{j})$ and $\mathbf{L}_{ij}=l(V_{i}, V_{j})$. $\mathbf{P}=\mathbf{I}-m^{-1}\mathbf{1}\mathbf{1}^\mathrm{T}\in\mathbb{R}^{m\times m}$ is a centering matrix, where $\mathbf{I}$ is an identity matrix and $\mathbf{1}$ is an all-one column vector. $\mathbf{P}$ is used to center the RBF kernel matrices to have zero mean in the feature space.

To eliminate the dependence between the high-level treatment representation with the corresponding confounders, sample reweighting techniques could generate a pseudo-distribution that has less dependence between variables~\cite{zhang2021deep,zou2020counterfactual}.  We propose a sample reweighting method to eliminate the dependence between high-level variables, where the non-linear dependence is measured by HSIC.
\par We use $\mathbf{w}\in \mathbb{R}^m_{+}$ to denote a set of sample weights. For any two random variables $U$ and $V$, we first utilize the random initialized weights to reweight these two variables:
\begin{equation}
    \hat{U}=(\mathbf{w}\cdot\mathbf{1}^\mathrm{T})\odot U,
\end{equation}
\begin{equation}
    \hat{V}=(\mathbf{w}\cdot\mathbf{1}^\mathrm{T})\odot V,
\end{equation}
where `$\odot$' is the Hadamard product. Substituting $\hat{U}$ and $\hat{V}$ into Eq.~(\ref{Eq:HSIC_0}), we obtain the weighted HSIC value:
\begin{equation}
    \hat{\text{HSIC}}^{k,l}_0(U,V,\mathbf{w})=(m-1)^{-2}tr(\mathbf{\hat{K}P\hat{L}P}),
\end{equation}
where $\hat{\mathbf{K}},\hat{\mathbf{L}}\in\mathbb{R}^{m\times m}$ are weighted RBF kernel matrices containing entries $\hat{\mathbf{K}}_{ij}=k(\hat{U}_{i}, \hat{U}_{j})$ and $\hat{\mathbf{L}}_{ij}=l(\hat{V}_{i}, \hat{V}_{j})$. Specifically, for treatment variable $\mathbf{H}_{,0:d}$ and its corresponding multiple confounder variables $\{\mathbf{H}_{,d:2d},\cdots,\mathbf{H}_{,(n_L-1)d:n_Ld}\}$, we share the sample weights $\mathbf{w}$ across multiple confounders and propose to optimize $\mathbf{w}$ by:
\begin{equation}
    \mathbf{w}^*=\mathop{\arg\min}_{\mathbf{w}\in\Delta_m}\sum_{1<p<n_L}\hat{\text{HSIC}}^{k,l}_0(\mathbf{H}_{,0:d},\mathbf{H}_{,(p-1)d:pd},\mathbf{w}),
\end{equation}
 where $\Delta_m=\{\mathbf{w}\in\mathbb{R}^m_{+}|\sum_{i=1}^m \mathbf{w}_i=m\}$ is used to control the overall loss of each batch almost unchange, and we utilize $\mathbf{w}=\text{softmax}(\mathbf{w})$ to satisfy this constraint. Hence, reweighting training samples with the optimal $\mathbf{w}^*$ can mitigate the dependence between high-level treatment variable with confounders to the greatest extent.

\paratitle{Global Multi-dimensional Variable Decorrelation.} Note that the above method is to remove the correlation between a single treatment variable $\mathbf{H}_{,0:d}$ with the confounders $\{\mathbf{H}_{,d:2d},\cdots,\mathbf{H}_{,(n_L-1)d:n_Ld}\}$. However, we need to estimate the causal effect of  all the learned high-level representations $\{\mathbf{H}_{,0:d},\mathbf{H}_{,d:2d},\cdots,\mathbf{H}_{,(n_L-1)d:n_Ld}\}$. As mentioned above, we need to set each high-level representation as a treatment variable and the remaining high-level representations as confounders, and remove the dependence between each treatment variable with the corresponding confounders. One effective way to achieve this goal is to remove all the dependence between variables. Specifically, we learn a set of sample weights that globally remove the dependence between each pair of high-level representations, defined as follows:
\begin{equation}
    \mathbf{w}^*=\mathop{\arg\min}_{\mathbf{w}\in\Delta_m}\sum_{1\leq i<j\leq n_L}\hat{\text{HSIC}}^{k,l}_0(\mathbf{H}_{,(i-1)d:id},\mathbf{H}_{,(j-1)d:jd},\mathbf{w}).
    \label{Eq:global}
\end{equation}
\par As we can see from Eq.~(\ref{Eq:global}), the global sample weights $\mathbf{w}$ simultaneously reduce the dependence among all high-level representations.

\subsection{Weighted Graph Neural Networks}
In the traditional GNN model, the parameters of the model are learned on the original graph dataset $\mathcal{G}=\{G_1,\cdots, G_m\}$. Because the sample weights $\mathbf{w}$ learned by the causal variable distinguishing regularizer are capable of globally decorrelating the high-level variables, we propose to use the sample weights to reweight the GNN loss, and iteratively optimize sample weights $\mathbf{w}$ and the parameters of weighted GNN model as follows:
\begin{equation}
    f^{(t+1)}, g^{(t+1)}=\mathop{\arg\min}_{f,g}\sum_{i=1}^m \mathbf{w}^{(t)}_i L(g(f(\mathbf{A}_i, \mathbf{F}_i)), y_i),
\label{Eq:gnn_loss}
\end{equation}

\begin{equation}
\small
\begin{aligned}
\mathbf{w}^{(t+1)}=\mathop{\arg\min}_{\mathbf{w}^{(t+1)}\in\Delta_m}\sum_{1\leq i<j\leq n_L}\hat{\text{HSIC}}^{k,l}_0(\mathbf{H}_{,(i-1)d:id}^{(t+1)},\mathbf{H}^{(t+1)}_{,(j-1)d:jd},\mathbf{w}^{(t)}),
\end{aligned}
\label{Eq:w_loss}
\end{equation}
where $f(\cdot)$ is the representation part of our model and its output is the high-level representation $\mathbf{H}$, $\mathbf{H}^{(t+1)}=f^{(t+1)}(\mathbf{A}, \mathbf{F})$, $t$ represents the iteration number, $g(\cdot)$ is a linear prediction layer, and $L(\cdot,\cdot)$ represents the loss function depends on which task we target. When updating the sample weights and the parameters of GNN model being fixed, we need to optimize the objective function Eq.~(\ref{Eq:w_loss}). We update the sample weights by mature optimization technique, i.e., Adam. After obtaining the sample weight of each graph in the batch, we optimize the objective function of weighted GNNs Eq.~(\ref{Eq:gnn_loss}) by an Adam optimizer~\cite{kingma2014adam}. For the classification task, cross-entropy loss is used, and for the regression task, 
least squared loss is used. Initially, $\mathbf{w}^{(0)}=(1,1,\cdots,1)^\mathrm{T}$ in each mini-batch. In the training phase, we iteratively optimize sample weights and model parameters with Eq.~(\ref{Eq:gnn_loss}) and (\ref{Eq:w_loss}). During the inference phase, the predictive model directly conducts prediction based on the GNN model without any calculation of sample weights. The detailed procedure of our model is shown in Algorithm~\ref{alg:stableGNN}.
\par Although StableGNN still performs on dataset $\mathcal{G}$, the weight $\mathbf{w}_i$ of each graph is no longer same. This weight adjusts the contribution of each graph in the mini-batch loss, so that the GNN parameters are learned on the dataset that each high-level features are decorrelated\footnote{In this paper, we slightly abuse the term "decorrelate/decorrelation", which means removing both the linear and non-linear dependence among features unless specified.} which can better learn the true correlation between relevant features and labels.

\begin{algorithm}
    \caption{Training process of StableGNN}
    \label{alg:stableGNN}
    \KwIn{Training graphs $\mathcal{G}_{train}=\{(G_1, y_1), \cdots,(G_N, y_N)\}$; 
    \\ Training Epoch:$Epoch$;
    \\ Decorrelation Epoch: $DecorEpoch$;}
    \KwOut{Learned GNN model;}
    \While{$t<Epoch$}{
    \For{1 \rm{to} \rm{BatchNumber}}
    {
      Forward propagation to generate $\mathbf{H}$\;
      \For{1 \rm{to} $DecorEpoch$}
      {
        Optimize sample weights $\mathbf{w}$ via Eq.~(\ref{Eq:w_loss})\;
      }
      Back propagate with weighted GNN loss to update $f$ and $g$ via Eq.~(\ref{Eq:gnn_loss})\;
    }}
\end{algorithm}
\paratitle{Discussions.} Our proposed framework, StableGNN, aims to relieve the distribution shifts problem by causal representation learning diagram in a general way. Specifically, a new causal representation learning for graphs that seamless integrates the power of representation learning and causal inference is proposed. For the representation learning part, we could utilize various state-of-the-art graph pooling layer~\cite{zhang2018end,ying2018hierarchical,lee2019self,gao2019graph} to extract high-level representations, nevertheless, the main intuition of our work is that we should learn high-level meaningful representation rather than meaningless mixed representation in our proposed learning diagram. This point is key for meaningful and effective causal learning, which is validated in Section~\ref{Sec::real-world}.

\paratitle{Limitations.} In our model, we assume a general causal variable relationships in Figure~\ref{fig:causal view}. Nevertheless, for some datasets or tasks, there may exist more complicated causal relationships among high-level variables, hence discovering causal structure for these high-level variables may be useful for reconstructing the latent data generation process and achieving better generalization ability.

\paratitle{Time Complexity Analysis.} As the proposed framework consists of two parts, we analyze them separately. For the graph high-level representation learning part, to cluster nodes, although it requires the additional computation of an assignment matrix, we observed that the Diffpool layer did not incur substantial additional running time in practice. The reason is that each DiffPool layer reduces the size of graphs by extracting a coarser high-level representation of the graph, which speeds up the graph convolution operation in the next layer. For CVD regularizer, given $n_L$ learned high-level representation, the complexity of computing HSIC value of each pair of high-level variables is $\mathcal{O}(m^2)$~\cite{song2012feature}, where $m$ is the batch size. Hence, for each batch, the computation complexity of CVD regularizer is $\mathcal{O}(tn_L(n_L-1)m^2)$, where $t$ is the number of epochs to optimize $\mathbf{w}$ and $n_L$ is a very small number.

\section{Experiments}
In this section, we evaluate our algorithm on both synthetic and real-world datasets, comparing with state-of-the-arts.

\paratitle{Model Configurations.} In our experiments, the GNN layer used for DiffPool layer is built on top of the GraphSAGE~\cite{hamilton2017inductive} or GCN~\cite{kipf2016semi} layer, to demonstrate that our framework can be applied on top of different GNN models. We use the ``max pooling'' variant of GraphSAGE.  One DiffPool layer is used for all the datasets and more sophisticated hierarchical layers could be learned by stacking multiple DiffPool layers. For our model, StableSAGE/StableGCN refers to using GraphSAGE/GCN as base model, respectively. All the models are trained with same learning rate mode and the model for prediction is selected based on the epoch with best validation performance.

\subsection{Experiments on Synthetic Data}
\label{sec::syn}
To better verify the advantages of StableGNN on datasets with different degrees of distribution shifts between training set and testing set, we generate the synthetic datasets with a clear generation process so that the bias degree of datasets is controllable.

\paratitle{Dataset.} We aim to generate a graph classification dataset that has a variety of distribution shifts from training dataset to testing dataset. Inspired by recent studies on GNN explanation~\cite{ying2019gnnexplainer,lin2021generative}, we utilize motif as a subgraph of graphs.  We first generate a base subgraph for each graph, that is, each positive graph has a ``house''-structured network motif and each negative graph has a motif that is randomly drawn from 4 candidate motifs (i.e., star, clique, diamond and grid motifs). Hence, the ``house'' motif is the causal structure that causally determines the label. To inject spurious correlation, $\mu * 100\%$ of positive graphs will be added ``star'' motif and the remaining positive and negative graphs will randomly add a non-causal motif from 4 candidate motifs. The node features are drawn from the same uniform distribution for all nodes. We set $\mu$ as \{0.6, 0.7, 0.8, 0.9\} to get four spurious correlation degrees for the training set. And we set $\mu=0.5$ to generate OOD validation set and $\mu=0.25$ to generate an unbiased testing dataset. The larger $\mu$ for the training set means there incurs a larger distribution shift between the training and testing sets. The resulting graphs are further perturbed by adding edges from the vertices of the first motif to one of the vertices of the second motif in each graph with the probability 0.25. The number of training samples is 2000, and for validation and testing set is 1000.

\paratitle{Experimental Settings and Metrics.} As the synthetic data has a known generating mechanism and our model is based on the GraphSAGE/GCN, to clearly validate and explain the effectiveness of our framework helping base GNN get rid of spurious correlation, in this subsection, we only compare with the base models. The number of layers of GraphSAGE and GCN is set as 5. The dropout rate for all the methods is set as 0.0. For all the GNN models, an initial learning rate is set as $1\times 10^{-3}$, the reduce factor is 0.5, and the patience value is 5. And the learning rate of CVD regularizer is selected from $\{0.1, 0.3, \cdots, 1.3\}$.  To baselines, the training epoch is set as 50. For Stable-SAGE/GCN, we set 20 epochs to warm up the parameters, i.e., training without the CVD regularizer, and 30 epochs to train the whole model. The decorrelation epoch to learn sample weights is set as 50. For all the methods, we take the model of the epoch with the best validation performance for prediction. The batch size is set as 250. The maximum number of clusters (i.e., high-level representations) is set as 7 for StableSAGE and 8 for StableGCN. For all the baseline models, if not mentioned specially, we aggregate node embeddings by mean pooling readout function to get the graph-level representation. Following~\cite{hendrycks2019benchmarking}, we augment each GNN layer with batch normalization (BN)~\cite{ioffe1502accelerating} and residual connection~\cite{he2016deep}. We evaluate the methods with three widely used metrics for binary classification, i.e., Accuracy, F1 score and ROC-AUC~\cite{huang2005using}. For all the experiments, we run 4 times with different random seeds and report their mean and standard error of prediction value with the corresponding metric on the test set in percent.

\begin{table*}[thbp]
\centering
\setlength\tabcolsep{20pt}
\caption{Results on synthetic datasets in different settings. The `improvements' means the improvement percent of StableSAGE/StableGCN against GraphSAGE/GCN.}
\begin{tabular}{ccccc}
\hline
 \multicolumn{1}{l}{Correlation Degree ($\mu$)} & Method    & Accuracy       & F1 score             & ROC-AUC                                 \\ \hline 
 \multirow{3}{*}{0.6}                           & GraphSAGE & 69.68$\pm$0.91 & 67.83$\pm$1.41  & 77.49$\pm$0.71                      \\ 
                                                                              & StableSAGE & 73.93$\pm$0.66 & 73.05$\pm$1.21  & \multicolumn{1}{c}{79.22$\pm$1.07} \\  
                                                                            & Improvements  & 6.10\%         & 7.70\%          & 2.23\%                              \\ \cline{1-5} 
                                  \multirow{3}{*}{0.7}                           & GraphSAGE & 67.85$\pm$0.76 & 63.76$\pm$1.32  & 75.55$\pm$1.01                      \\ 
                                                                                & StableSAGE & 73.9$\pm$1.78  & 71.23$\pm$1.88  & 81.48$\pm$4.13                      \\ 
                                                                                 & Improvements  & 8.92\%         & 11.71\%         & 7.85\%                              \\ \cline{1-5} 
                                  \multirow{3}{*}{0.8}                           & GraphSAGE & 65.67$\pm$1.22 & 60.23$\pm$0.81  & 72.70$\pm$1.45                      \\ 
                                                                                & StableSAGE & 72.15$\pm$1.26 & 68.56$\pm$0.87  & 81.35$\pm$2.12                      \\ 
                                                                                & Improvements  & 9.86\%         & 13.83\%         & 9.98\%                              \\ \cline{1-5}             
                                  \multirow{3}{*}{0.9}                           & GraphSAGE & 65.2$\pm$0.94  & 54.24$\pm$1.98  & 72.89$\pm$0.67                      \\ 
                                                                                & StableSAGE & 70.35$\pm$1.66 & 64.84$\pm$2.54 & 80.31$\pm$1.78                      \\ 
                                                                                & Improvements  & 7.90\%         & 19.54\%         & 10.18\%                             \\ \hline                             \hline
\multirow{3}{*}{0.6}                           & GCN       & 70.98$\pm$0.93 & 67.06$\pm$2.95  & 77.55$\pm$0.55                      \\ 
                                                                             & StableGCN & 74.92$\pm$1.91 & 73.91$\pm$2.49  & 81.79$\pm$2.42                      \\ 
                                                                              & Improvements  & 5.56\%         & 10.21\%         & 5.47\%                              \\ \cline{1-5} 
                                 \multirow{3}{*}{0.7}                           & GCN       & 70.9$\pm$1.45  & 65.57$\pm$3.69  & 78.27$\pm$1.53                      \\ 
                                                                            & StableGCN & 73.15$\pm$2.62 & 70.29$\pm$2.76  & 79.77$\pm$3.42                      \\ 
                                                                                & Improvements  & 3.17\%         & 7.198\%         & 1.92\%                              \\ \cline{1-5} 
                                 \multirow{3}{*}{0.8}                           & GCN       & 70.35$\pm$0.50 & 65.41$\pm$0.85  & 75.76$\pm$1.09                      \\ 
                                                                               & StableGCN & 74.5$\pm$1.03  & 70.94$\pm$1.69  & 81.53$\pm$0.86                      \\ 
                                                                             & Improvements  & 5.90\%         & 8.45\%          & 7.62\%                              \\ \cline{1-5} 
                                 \multirow{3}{*}{0.9}                           & GCN       & 69.68$\pm$0.56 & 62.28$\pm$1.38  & 76.61$\pm$0.66                      \\ 
                                                                               & StableGCN & 76.35$\pm$1.37 & 72.64$\pm$2.62  & 83.24$\pm$0.58                      \\ 
                                                                                & Improvements  & 9.57\%         & 16.63\%         & 8.65\%                              \\ \hline
\end{tabular}

\label{Table:syn}
\end{table*}

\paratitle{Results on Synthetic Data.} The results are given in Table~\ref{Table:syn}, and we have the following observations. First, both the GraphSAGE and GCN  suffer from serious performance decrease with the increase of spurious correlation degree, e.g., for F1 score, GraphSAGE drops from 67.83 to 54.24, and GCN drops from 67.06 to 62.28, indicating that spurious correlation greatly affects the GNNs' generalization performance and the heavier distribution shifts will cause a larger performance decrease. Second, compared with the base model, our proposed models achieve up to 19.64\% performance improvements, and gain larger improvements under heavier distribution shifts. As we know the "house" motif is decisive for the label, the only way to improve the performance is to utilize this causal subgraph, demonstrating that our models could significantly reduce the influence of spurious correlation among subgraphs and reveal the true relationship between causal subgraphs with labels. Third, when building our framework both on GraphSAGE and GCN, our framework could achieve consistent improvements, and it indicates that StableGNN is a general framework and has the potential to adapt to various GNN architectures.

\paratitle{Explanation Analysis.} An intuitive type of explanation for GNN models is to identify subgraphs that have a strong influence on final decisions~\cite{ying2019gnnexplainer}. To demonstrate whether the model focuses on the relevant or irrelevant subgraphs while conducting prediction, we utilize GNNExplainer~\cite{ying2019gnnexplainer} to calculate the most important subgraph with respect to GNN's prediction and visualize it with red color. As GNNExplainer needs to compute an edge mask for explanation and GraphSAGE cannot be aware of the edge weights, we explain the GCN-based model, i.e., GCN and StableGCN. As shown in Figure~\ref{fig:syn_explainer}, we find the following interesting cases contributing to the success of our model, where each case represents a kind of failure easily made by existing methods. 
\begin{itemize}
    \item \textbf{Case 1.} As we can see, GCN assigns the higher weights to ``star'' motif, however, StableGCN concentrates more on ``house'' motif. Although GCN could make correct predictions based on the ``star'' motif, which is highly correlated with ``house'' motif, this prediction is unstable and it means that if there incurs spurious correlations in the training set, the model could learn the spurious correlation and rely on this clue for predictions. This unstable prediction is undesirable, as the unstable structure may change during the wild testing environments.
    \item \textbf{Case 2.} In this case, there is a ``grid'' motif connecting with ``house'' motif. As we can see, GCN pays more attention on ``grid'' motif and StableGCN still concentrates on ``house'' motif. Due to the existence of spurious correlated subgraphs, it will reduce the confidence of the true causal subgraph for prediction. When there appears another irrelevant subgraph, GCN may also pay attention on this irrelevant subgraph, leading to incorrect prediction results. However, our model could focus on the true causal subgraphs regardless of which kind of subgraphs are associated with them.
    
    \item \textbf{Case 3.} This is a case for negative samples. The spurious correlation leads to the GCN model focusing on the ``star'' motif. As the ``star'' motif is correlated with the positive label, GCN model will predict this graph as a positive graph. In contrast, due to the decorrelation of subgraphs in our model, we find that the ``star'' motif may not discriminate to the label decision and ``diamond'' motif may attribute more to the negative labels.
     
\end{itemize}

\begin{figure}[t]
	\centering
	\includegraphics[width=9cm]{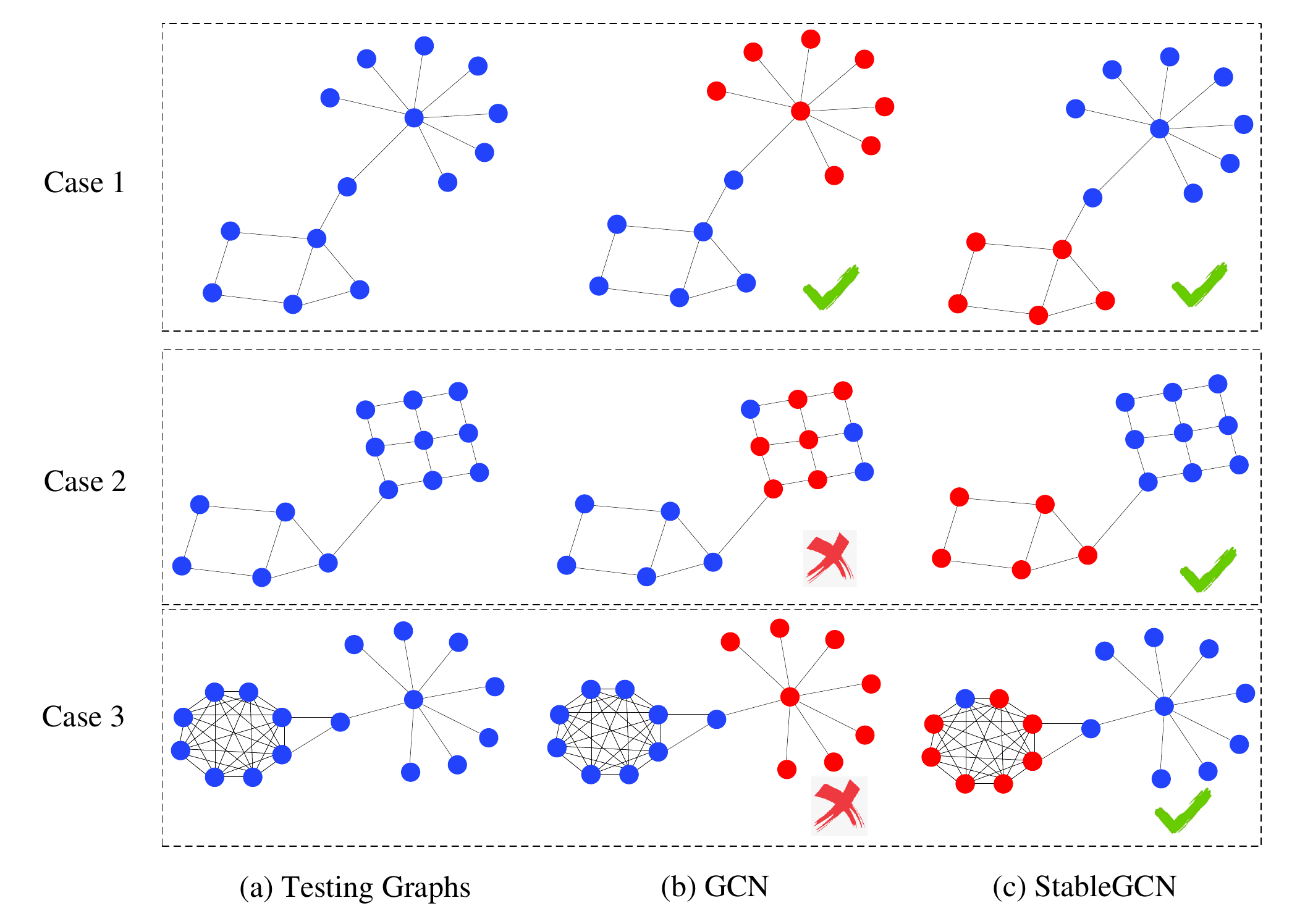}
	\caption{Explanation cases of GCN and StableGCN. Red nodes are the important subgraph calculated by the GNNExplainer.}
	\label{fig:syn_explainer}
\end{figure}
\subsection{Experiments on Real-world Datasets}
\label{Sec::real-world}
In this section, we apply our StableGNN algorithm on eight real-world datasets for out-of-distribution graph property prediction.

\paratitle{Datasets.} We adopt seven molecular property prediction datasets from OGB Datasets~\cite{hu2020open}.  All the molecules are pre-processed using RDKit~\cite{landrum2006rdkit}. Each graph represents a molecule, where nodes are atoms, and edges are chemical bonds. We use the 9-dimensional node features and 3-dimensional edges features provided by ~\cite{hu2020open}, which has better generalization performance. The task of these datasets is to predict the target molecular properties, e.g., whether a molecule inhibits HIV virus replication or not. Input edge features are 3-dimensional,
containing bond type, bond stereochemistry as well as an additional bond feature indicating whether
the bond is conjugated. Depending on the properties of molecular, these datasets can be categorized into three kinds of tasks: binary classification, multi-label classification and regression. Different from commonly used random splitting, these datasets adopt a scaffold splitting~\cite{wu2018moleculenet} procedure that splits the molecules with different scaffolds into training or testing sets. All the molecules with the same scaffold could be treated as an environment, and the scaffold splitting attempts to separate molecules with different scaffolds into different subsets. For example, two molecules, Cyclopropanol (C3H6O) and 1,4-Cyclohexanediol (C6H12O2), contain different scaffold patterns: the former scaffold is 3C-ring and the latter is 6C-ring. Although sampled with different distributions, they are both readily soluble in water due to the invariant subgraph hydroxy (-OH) attached to different scaffolds~\cite{ishida2021graph}. The environments of training and testing sets are different, resulting in different graph distributions. Therefore, this kind of data could be used to test whether the model could leverage the causal subgraph to make predictions, i.e., the generalization ability of GNNs on graphs with different distributions. Moreover, we also conduct the experiments on a commonly used graph classification dataset, MUTAG~\cite{debnath1991structure}, as we could explain the results based on the knowledge used in~\cite{luo2020parameterized}. It consists of 4,337 molecule graphs. Each graph is assigned to one of 2 classes based on its mutagenic effect. Note that this dataset cannot adopt the scaffold splitting, as the +4 valence \textit{N} atom, which commonly exists in this dataset, is illegal in the RDKit~\cite{landrum2013rdkit} tool used for scaffold splitting. We just use the random splitting for this dataset, however, we still believe there are some OOD cases in the testing set. The splitting ratio for all the datasets is 80/10/10. The detailed statistics are shown in Table~\ref{Tab:Datasets}.

\paratitle{Baselines.}  As the superiority of GNNs against traditional methods on graph-level tasks, like kernel-based methods~\cite{shervashidze2011weisfeiler,shervashidze2009efficient}, has been demonstrated by previous literature~\cite{ying2018hierarchical}, here, we mainly consider baselines based upon several related and state-of-the-art GNNs. 
\begin{itemize}
    \item Base models: GraphSAGE~\cite{hamilton2017inductive} and GCN~\cite{kipf2016semi} are classical GNN methods. We utilize them as base models in our framework, so they are the most related baselines to validate the effectiveness of the proposed framework. 
    \item DIFFPOOL~\cite{ying2018hierarchical}: It is a hierarchical graph pooling method in an end-to-end fashion. It adopts the same model architecture and hyperparamter setting with high-level variable learning components in our framework, except that DIFFPOOL model aggregates the clusters' representation by summation operation rather than an ordered concatenation operation to generate the final high-level representations.
    \item GAT~\cite{velivckovic2017graph}: It is an attention-based GNN method, which learns the edge weights by an attention mechanism.
    \item GIN~\cite{xu2018how}: It is a graph isomorphism network that is as powerful as the WL graph isomorphism test. We compare with two variants of GIN, i.e., whether the $\epsilon$ is a learnable parameter or a fixed 0 scalar, denoted as GIN and GIN0, respectively. 
    \item MoNet~\cite{monti2017geometric}: It is a general architecture to learn on graphs and manifolds using the bayesian gaussian mixture model.
    \item SGC~\cite{pmlr-v97-wu19e}: It is a simplified GCN-based method, which reduces the excess complexity through successively removing nonlinearities and collapsing weight matrices between consecutive layers.
    \item JKNet~\cite{xu2018representation}: It is a GNN framework that selectively combines different aggregations at the last layer.
    \item DIR~\cite{wu2022discovering}: It is a GNN method designed for the distribution shift problem, which disentangles the casual and non-causal subgraphs.
    \item CIGA~\cite{chen2022learning}: It is a GNN method that learns causally invariant representations for OOD generalization on graphs.
\end{itemize}

\paratitle{Experimental Settings and Metrics.} Here, we only describe the experimental settings that are different from Section~\ref{sec::syn}. For all GNN baselines, we follow~\cite{hu2020open} which set the number of layers as 5. The dropout rate after each layer for all the methods on the datasets for binary classification and multi-label classification is set as 0.5, and for regression datasets, the dropout rate is set as 0.0. The number of batch size for binary classification and multi-label classification is set as 128, and for regression datasets, the batch size is set as 64. The training epoch for all the baselines is set as 200. And for our model, we utilize 100 epochs to warm up and 100 epochs to train the whole model. The maximum number of clusters is set as 7 for our models and DIFFPOOL. For DIR and CIGA, the hyperparameter $s_c$ the selection ratio of casual subgraph is chosen from \{0.1, 0.2, 0.25, 0.3, 0.4, 0.5, 0.6, 0.7, 0.8, 0.9\}, and the hyperparameters $\alpha$ and $\beta$ for contrastive loss and hinge loss of CIGA are both chosen from \{0.5, 1, 2, 4, 8, 16, 32\} according to the validation performances.  Moreover, as the loss of DIR and CIGA are specifically designed for the multi-classification task, they cannot perform on the datasets with multi-label classification and regression tasks. As these datasets usually treat chemical bond type as their edge type, to include edge features, we follow~\cite{hu2020open} and add transformed edge features into the incoming node features. For the datasets from OGB, we use the metrics recommended by the original paper for each task~\cite{hu2020open}. And following~\cite{ying2019gnnexplainer}, we adopt Accuracy for MUTAG dataset. 

\begin{table*}[]
\setlength\tabcolsep{10pt}
\caption{Summary of real-world datasets.}
\begin{tabular}{cccccccc}
\hline
Dataset     & Task Type & \#Graphs & \begin{tabular}[c]{@{}c@{}}Average\\ \#Nodes\end{tabular} & \begin{tabular}[c]{@{}c@{}}Average\\ \#Edges\end{tabular} & \#Task                   & Splitting Type     & Metric   \\ \hline

Molbace   & Binary classification    & 1,513    & 34.1                                                      & 36.9                                                      & 1           & Scaffold splitting & ROC-AUC  \\ 
Molbbbp    & Binary classification   & 2,039    & 24.1                                                      & 26.0                                                      & 1           & Scaffold splitting & ROC-AUC  \\
Molhiv  & Binary classification     & 41,127   & 25.5                                                      & 27.5                                                      & 1            & Scaffold splitting & ROC-AUC  \\ 
MUTAG  & Binary classification  & 4,337    & 30.32                                                     & 30.77                                                     & 1           & Random splitting   & Accuracy \\ 
Molclintox   & Multi-label classification & 1,477    & 26.2                                                      & 27.9                                                      & 2       & Scaffold splitting & ROC-AUC  \\
Moltox21   & Multi-label classification  & 7,831    & 18.6                                                      & 19.3                                                      & 12      & Scaffold splitting & ROC-AUC  \\ 
Molesol   & Regression   & 1,128    & 13.3                                                      & 13.7                                                      & 1                       & Scaffold splitting & RMSE     \\ 
Mollipo  & Regression   & 4,200    & 27.0                                                      & 29.5                                                      & 1                       & Scaffold splitting & RMSE     \\ \hline
\end{tabular}
\label{Tab:Datasets}
\end{table*}

\paratitle{Results on Real-world Datasets.} The experimental results on eight datasets are presented in Table~\ref{Table:Real-results}, and we have the following observations. First, comparing with these competitive GNN models, Stable-SAGE/GCN achieves 6 rank one and 2 rank two on all eight datasets. And the average rank of StableSAGE and StableGCN are 1.75 and 2.87, respectively, which is much higher than the third place, i.e., 4.38 for MoNet. It means that most existing GNN models cannot perform well on OOD datasets and our model significantly outperforms existing methods, which well demonstrates the effectiveness of the proposed causal representation learning framework. Second, compared with the base models, i.e., GraphSAGE and GCN, our models achieve consistent improvements on all datasets, validating that our framework could boost the existing GNN architectures. Third, Stable-SAGE/GCN also outperforms DIFFPOOL method by a large margin. Although we utilize the DiffPool layer to extract high-level representations in our framework, the seamless integration of representation learning and causal learning is the key to the improvements of our model. Fourth, Stable-SAGE/GCN outperforms DIR and CIGA, showing the effectiveness of our model over them on OOD generalization problem on graphs. And DIR and CIGA need to set a hyperparameter for the selection ratio of casual subgraph for all graphs, where the selection ratio is fixed and the same for all graphs and it is hard to set in practice. Fifth, Stable-SAGE/GCN achieves superior performance on datasets with three different tasks and a wide range of dataset scales, indicating that our proposed framework is general enough to be applied to datasets with a variety of properties. 
\begin{table*}[]
\caption{Performance of real-world graph datasets. The number in the $(\cdot)$ along with each performance number means the rank of the method among all the methods on this dataset. Because the losses of DIR and CIGA are designed for binary/multi-classifcation task, they cannot perform on the datasets with multilabel classification and regression tasks. Hence, we only show the results of them on binary classification task and do not rank them. ``$\uparrow$'' means that for this metric, the larger value means better performance. ``$\downarrow$'' means that for this metric, the smaller value means better performance. Best results of
all methods are indicated in bold.}\resizebox{.98\textwidth}{!}{
\setlength\tabcolsep{15pt}
\centering
\begin{tabular}{c|ccccc}
\hline
\multirow{3}{*}{Methods}  & \multicolumn{4}{c}{Binary Classification}                                                     &              \\ 
          & Molhiv                 & Molbace               & Molbbbp               & MUTAG                 &              \\ 
          & ROC-AUC  ($\uparrow$)               & ROC-AUC  ($\uparrow$)              & ROC-AUC  ($\uparrow$)              & Accuracy  ($\uparrow$)             &              \\ \hline
 
GIN       & 76.21$\pm$0.53 (6)      & 74.50$\pm$2.75 (9)     & 67.72$\pm$1.89 (4)     & 78.86$\pm$1.15 (10)     &              \\ 
GIN0      & 75.49$\pm$0.91 (9)      & 74.36$\pm$3.48 (10)     & 66.65$\pm$1.32 (6)     & 79.03$\pm$0.56 (9)     &              \\ 
GAT       & 76.52$\pm$0.69 (5)      & \textbf{81.17$\pm$0.8 (1)}      & 67.17$\pm$0.49 (5)     & 79.26$\pm$1.13 (8)     &              \\ 
MoNet     & 77.13$\pm$0.79 (3)      & 76.92$\pm$0.91 (6)     & \textbf{69.52$\pm$0.46 (1)}    & 79.95$\pm$0.86 (4)     &              \\ 
SGC       & 69.46$\pm$1.44 (11)     & 71.28$\pm$1.79 (11)    & 61.17$\pm$2.91 (11)     & 69.53$\pm$0.77 (11)    &              \\ 
JKNet     & 74.99$\pm$1.60 (10)      & 78.99$\pm$13.4 (3)     & 65.62$\pm$0.77 (9)     & 79.49$\pm$1.16 (7)     &              \\ 
DIFFPOOL  & 75.75$\pm$1.38 (8)      & 74.69$\pm$11.13 (8)    & 63.35$\pm$2.21 (10)     & 80.13$\pm$1.32 (3)     &              \\\hline
DIR & 61.40$\pm$10.69 & 60.81$\pm$10.47 & 54.65$\pm$8.08&62.61$\pm$3.00 &\\
CIGA & 60.24$\pm$9.13 & 72.52$\pm$8.43 & 60.06$\pm$5.66&74.12$\pm$6.19 &\\\hline
GCN       & 76.63$\pm$1.13 (4)      & 75.88$\pm$1.85 (7)     & 66.47$\pm$0.90 (7)     & 79.89$\pm$1.32 (5)     &              \\
StableGCN       & \textbf{77.79 $\pm$1.19 (1)}      & 76.95$\pm$3.27 (5)     & 68.82$\pm$3.87 (2)     & 81.45$\pm$2.00 (2)     &              \\ \hline
GraphSAGE & 75.78$\pm$2.19 (6)      & 78.51$\pm$1.72 (4)     & 66.16$\pm$0.97 (8)     & 79.78$\pm$0.78 (6)     &              \\ 
StableSAGE & 77.63$\pm$0.79 (2)      & 80.73$\pm$3.98 (2)     & 68.47$\pm$2.47 (3)     & \textbf{82.13$\pm$0.32 (1)}     &              \\ \hline
\hline
\multirow{3}{*}{Methods}  & \multicolumn{2}{c}{Multilabel Classification} & \multicolumn{2}{c}{Regression}               &  \multirow{3}{*}{Average-Rank} \\ 
          & Molclintox             & Moltox21              & Molesol               & Molipo                &              \\ 
          & ROC-AUC   ($\uparrow$)      & ROC-AUC  ($\uparrow$)      & RMSE ($\downarrow$)                 & RMSE ($\downarrow$)                 &   ($\downarrow$)           \\ \hline

GIN       & 86.86$\pm$3.78 (6)      & 64.20$\pm$0.23 (11)    & 1.1002$\pm$0.0450 (4)  & 0.8051$\pm$0.0323 (10)  & 7.50 (9)         \\ 
GIN0      & 89.31$\pm$2.11 (3)      & 64.62$\pm$0.87 (10)     & 1.1358$\pm$0.0587 (5)  & 0.8050$\pm$0.0123 (9)  & 7.63 (10)       \\ 
GAT       & 83.47$\pm$1.37 (8)      & 68.81$\pm$0.48 (5)     & 1.2758$\pm$0.0269 (10)  & 0.8101$\pm$0.0183 (9)  & 6.38 (7)        \\ 
MoNet     & 86.75$\pm$1.22 (7)      & 67.02$\pm$0.26 (7)     & 1.0753$\pm$0.0357 (3)  & 0.7379$\pm$0.0117 (4)  & 4.38 (3)        \\ 
SGC       & 77.76$\pm$1.87 (10)     & 66.49$\pm$1.10 (8)     & 1.6548$\pm$0.0462 (11) & 1.0681$\pm$0.0148 (10) & 10.38 (11)        \\ 
JKNet     & 81.63$\pm$2.79 (9)      & 65.98$\pm$0.46 (9)     & 1.1688$\pm$0.0434 (7)  & 0.7493$\pm$0.0048 (5)   & 7.38 (8)        \\ 
DIFFPOOL  & 90.48$\pm$2.42 (2)      & 69.05$\pm$0.94 (3)     & 1.176$\pm$0.01388 (8)  & 0.7325$\pm$0.0221 (2)   & 5.50 (4)       \\ \hline
GCN       & 86.23$\pm$2.81 (7)      & 67.75$\pm$0.66 (6)     & 1.153$\pm$0.0392 (6)     & 0.7927$\pm$0.0086 (8)  & 6.25 (6)       \\ 
StableGCN       & 87.98$\pm$2.37 (5)      & \textbf{70.80$\pm$0.31 (1)}     & \textbf{0.9638$\pm$0.0292 (1)}     & 0.7839$\pm$0.0165 (6)  & 2.87 (2)       \\ \hline
GraphSAGE & 88.60$\pm$2.44 (4)      & 68.88$\pm$0.59 (4)     & 1.1852$\pm$0.0353 (9)  & 0.7911$\pm$0.0147 (7)  & 6.00 (5)       \\ 
StableSAGE & \textbf{90.96$\pm$1.93 (1)}      & 69.14$\pm$0.24 (2)     & 1.0092$\pm$0.0706 (2)  & \textbf{0.6971$\pm$0.0297 (1)}  & \textbf{1.75 (1)}        \\ \hline
\end{tabular}}
\label{Table:Real-results}
\end{table*}

\paratitle{Ablation Study.} Note that our framework naturally incorporates high-level representation learning and causal effect estimation in a unified framework. Here we conduct ablation studies to investigate the effect of each component. For our framework without the CVD regularizer, we term it as StableGNN-NoCVD. \footnote{Note that, for simplicity, in the following studies we mainly conduct analysis on StableSAGE, and StableGCN will get similar results. StableGNN will refer to StableeSAGE unless mentioned specifically.} The results are presented in Figure~\ref{fig:ablation}. We first find that StableGNN-NoCVD outperforms GraphSAGE on most datasets, indicating that learning hierarchical structure by our model for graph-level tasks is necessary and effective. Second, StableGNN-NoCVD achieves competitive results or better results with DIFFPOOL method. The only difference between them is that DIFFPOOL model averages the learned clusters' representations and StableGNN-NoCVD concatenates them by a consistent order, and then the aggregated embeddings are fed into a classifier. As the traditional MLP classifier needs the features of all samples should in a consistent order, the superior performance of StableGNN-NoCVD well validates that the learned representations $\mathbf{H}$ are in a consistent order across graphs. Moreover, we find that StableGNN consistently outperforms StableGNN-NoCVD. As the only difference between the two models is the CVD term, we can safely attribute the improvements to the distinguishing of causal variables by our proposed regularizer. Note that on some datasets, we could find that StableGNN-NoDVD cannot achieve satisfying results, however, when combined with the regularizer, it makes clear improvements. This phenomenon further validates the necessity of each component that we should conduct representation learning and causal inference jointly.

\begin{figure*}[!htbp]
\centering
\subfigure[Binary Classification ($\uparrow$)]{
\includegraphics[ width=5cm]{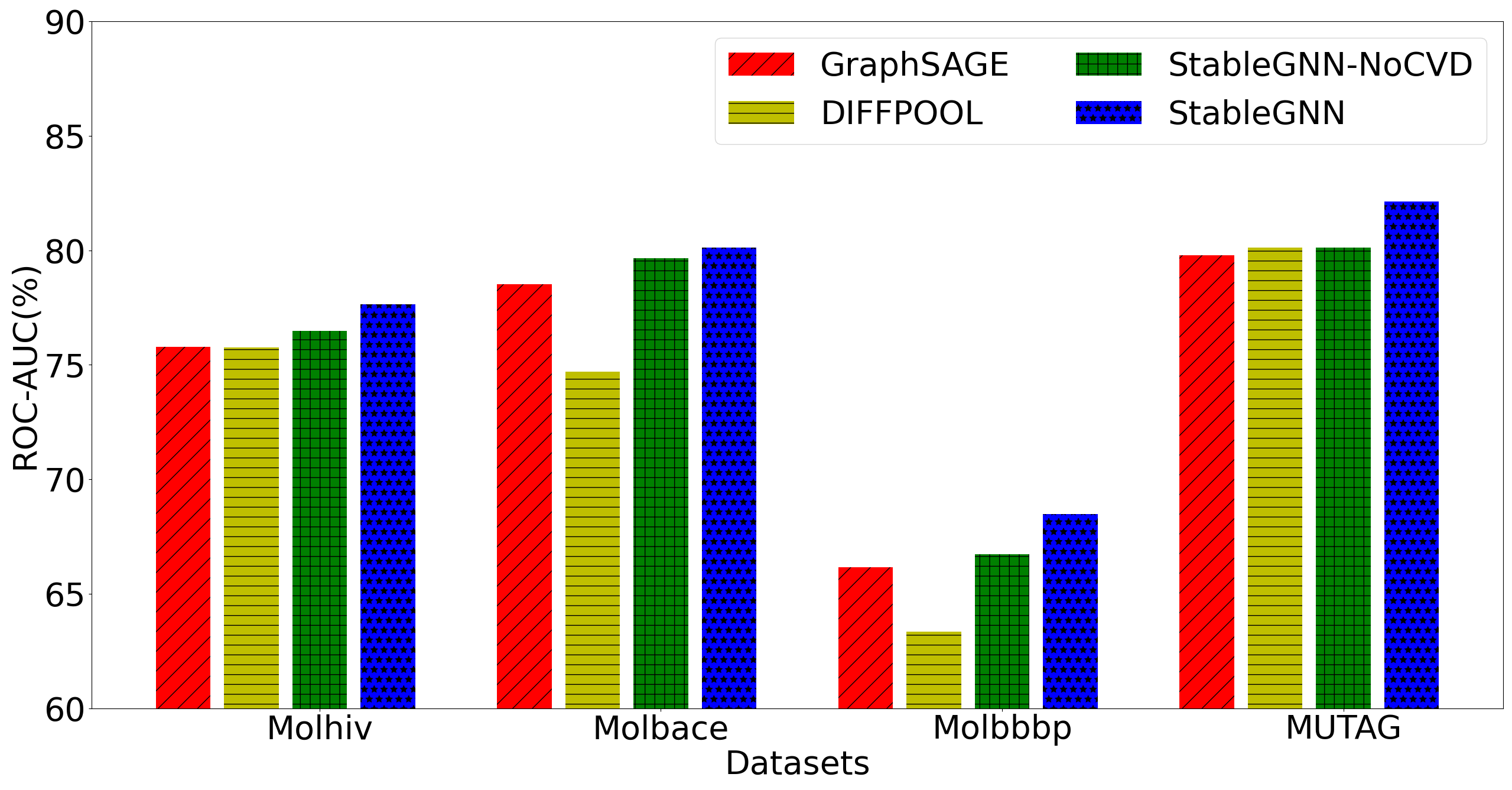}
\label{fig:gradient}
}
\subfigure[Multilabel Classification ($\uparrow$)]{
\includegraphics[ width=5cm]{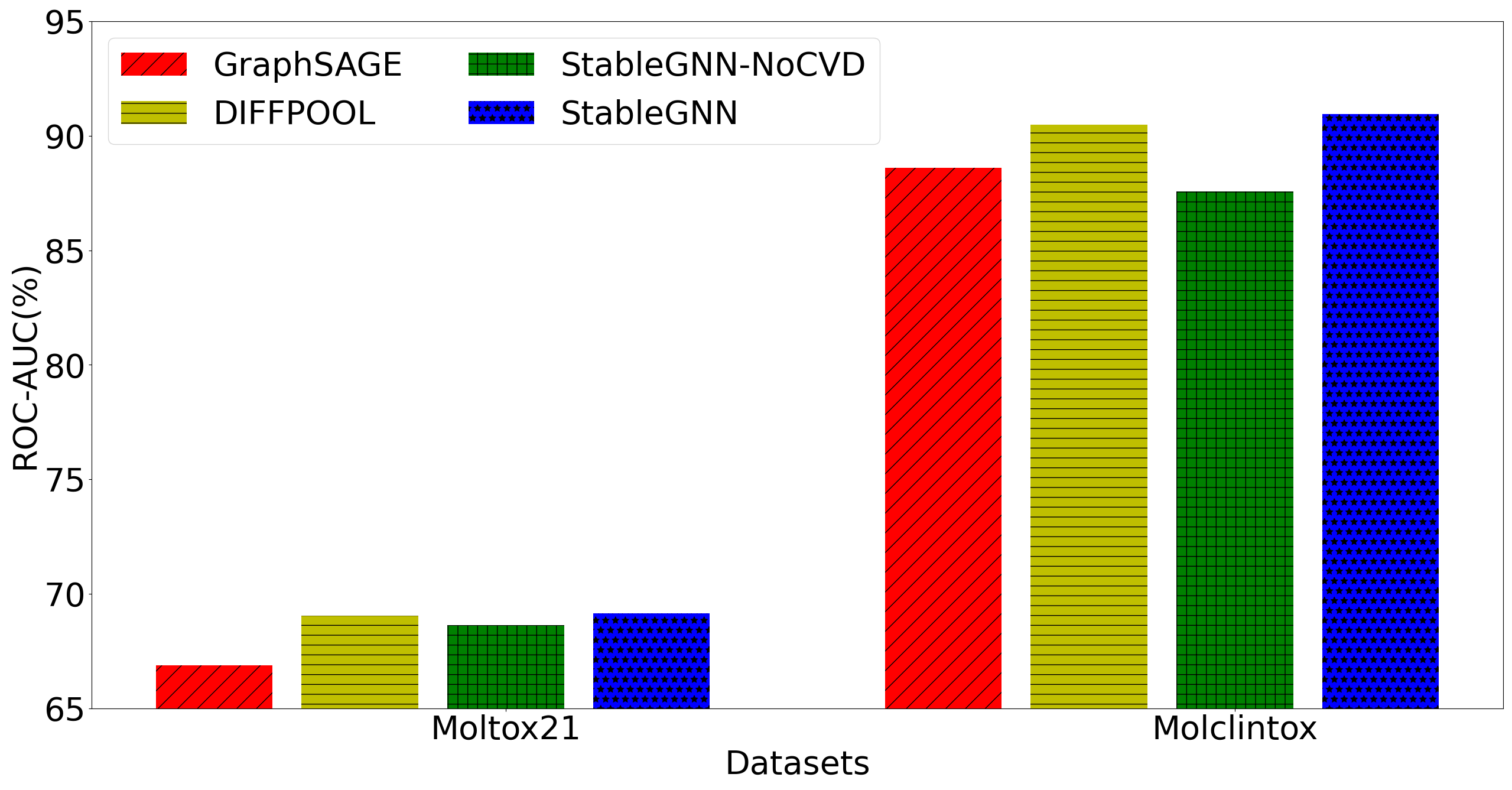}
\label{fig:gradient}
}
\subfigure[Regression ($\downarrow$)]{
\includegraphics[width=5cm]{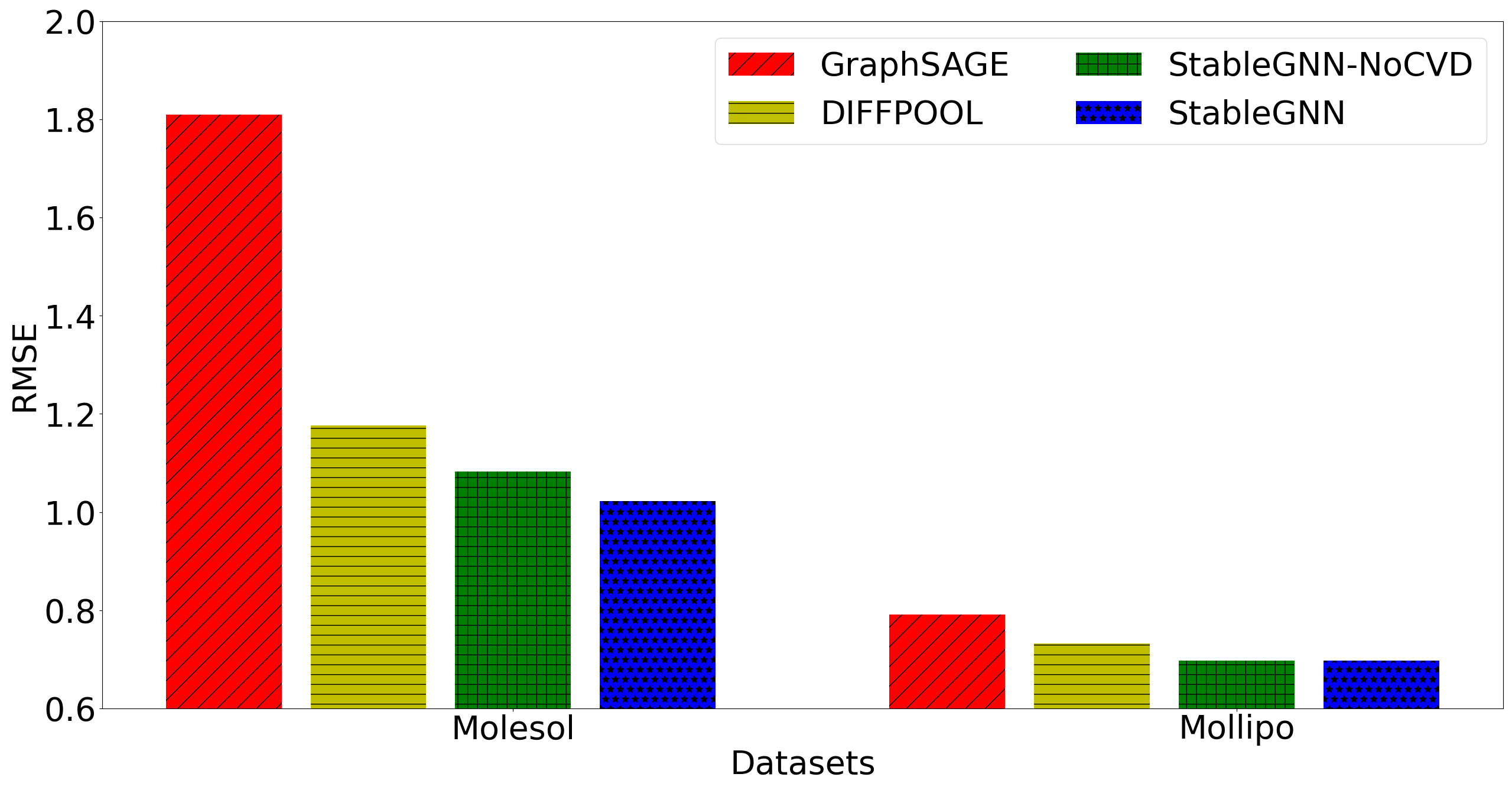}
\label{fig:gradient}
}
\caption{Ablation study on three real-world tasks.}
\label{fig:ablation}
\end{figure*}
\begin{table*}[]
\caption{Performance of different decorrelation methods.}\resizebox{.98\textwidth}{!}{
\begin{tabular}{c|cccccccc}
\hline
Method                 & Molhiv ($\uparrow$)        & Molbace ($\uparrow$)      & Molbbbp ($\uparrow$)       & MUTAG  ($\uparrow$)        & Molclintox ($\uparrow$)    & Moltox21 ($\uparrow$)      & Molesol ($\downarrow$)          & Mollipo ($\downarrow$)          \\ \hline
GraphSAGE              & 75.78$\pm$2.19 & 78.51$\pm$1.72 & 66.16$\pm$0.97 & 79.78$\pm$0.7  & 88.60$\pm$2.44 & 68.88$\pm$0.59 & 1.8098$\pm$0.1220 & 0.7911$\pm$0.0147 \\
GraphSAGE-Decorr       & 76.52$\pm$0.69 & 77.34$\pm$5.63 & 66.28$\pm$0.89 & 80.76$\pm$1.32 & 86.51$\pm$0.82 & 68.77$\pm$0.66 & 1.7889$\pm$0.1234 & 0.8024$\pm$0.0165 \\ \hline
StableGNN-NoCVD        & 76.47$\pm$1.01 & 79.65$\pm$0.86 & 66.73$\pm$1.87 & 80.13$\pm$1.80 & 87.56$\pm$1.91 & 68.63$\pm$0.63 & 1.0819$\pm$0.0219 & 0.6971$\pm$0.017  \\
StableGNN-NoCVD-Decorr & 75.32$\pm$0.52 & 78.71$\pm$2.34 & 67.02$\pm$1.55 & 79.17$\pm$1.07 & 88.89$\pm$2.58 & 69.05$\pm$0.66 & 1.0377$\pm$0.0389 & 0.7171$\pm$0.0378 \\ \hline
StableGNN-NoCVD-Distent & 75.76$\pm$0.64 & 79.04$\pm$1.24 & 64.21$\pm$1.34 & 80.92+0.50 & 90.44$\pm$0.79 & 68.87$\pm$1.28 & 1.0729$\pm$0.0256 & 0.7171$\pm$0.0378 \\ \hline
StableGNN      & \textbf{77.63$\pm$0.79} & \textbf{80.73$\pm$3.98} & \textbf{68.47$\pm$2.47} & \textbf{82.13$\pm$0.32} & \textbf{90.96$\pm$1.93} & \textbf{69.14$\pm$0.24} & \textbf{1.022$\pm$0.0039}  & \textbf{0.6971$\pm$0.0297} \\ \hline
\end{tabular}}
\label{Tab::Difference decorrelation}
\end{table*}

\begin{figure*}[htbp]
\centering
\subfigure[Molbace]{
\includegraphics[ width=5.5cm]{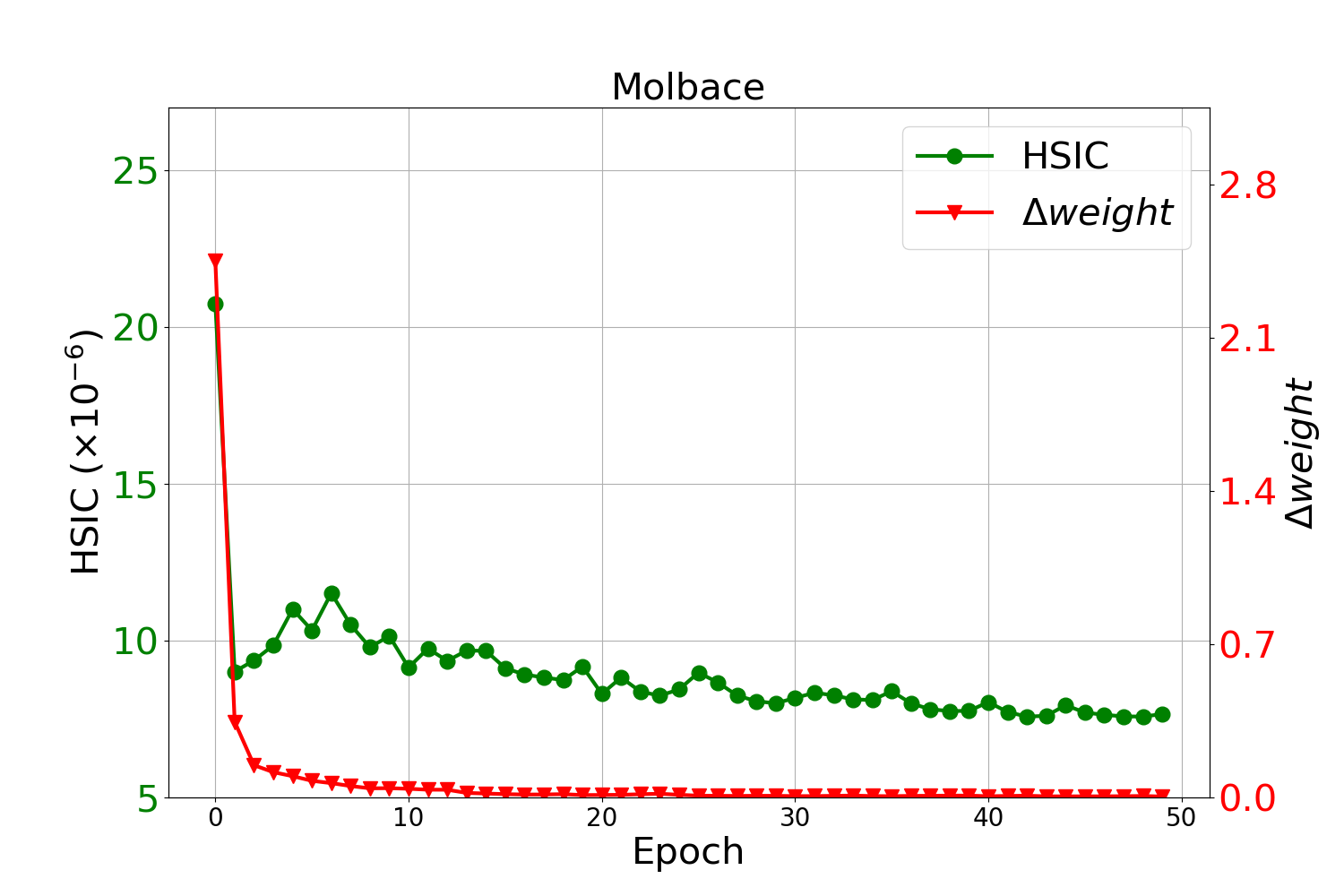}
\label{fig:gradient}
}
\subfigure[Molbbbp]{
\includegraphics[ width=5.5cm]{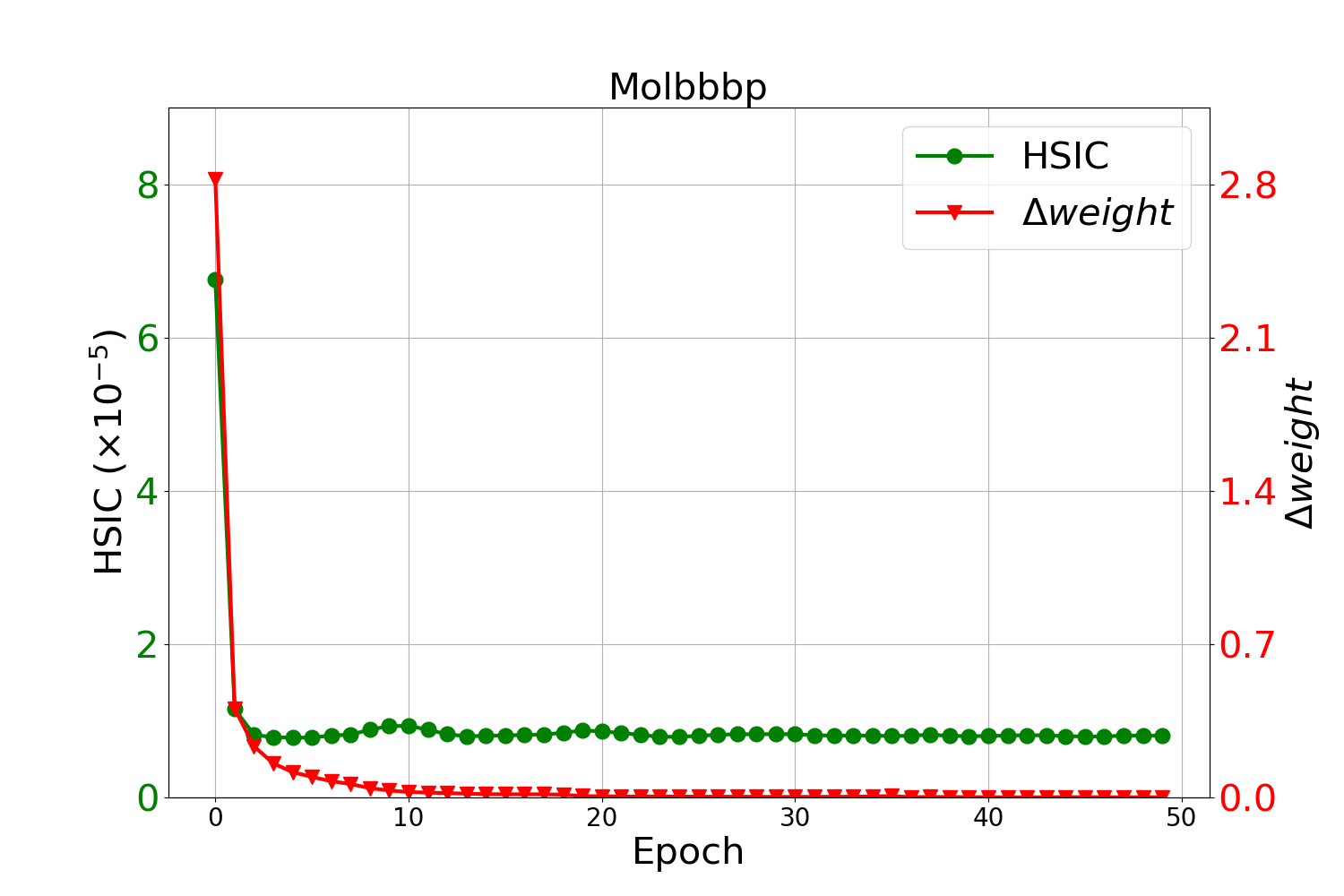}
\label{fig:gradient}
}
\subfigure[Molclintox]{
\includegraphics[width=5.5cm]{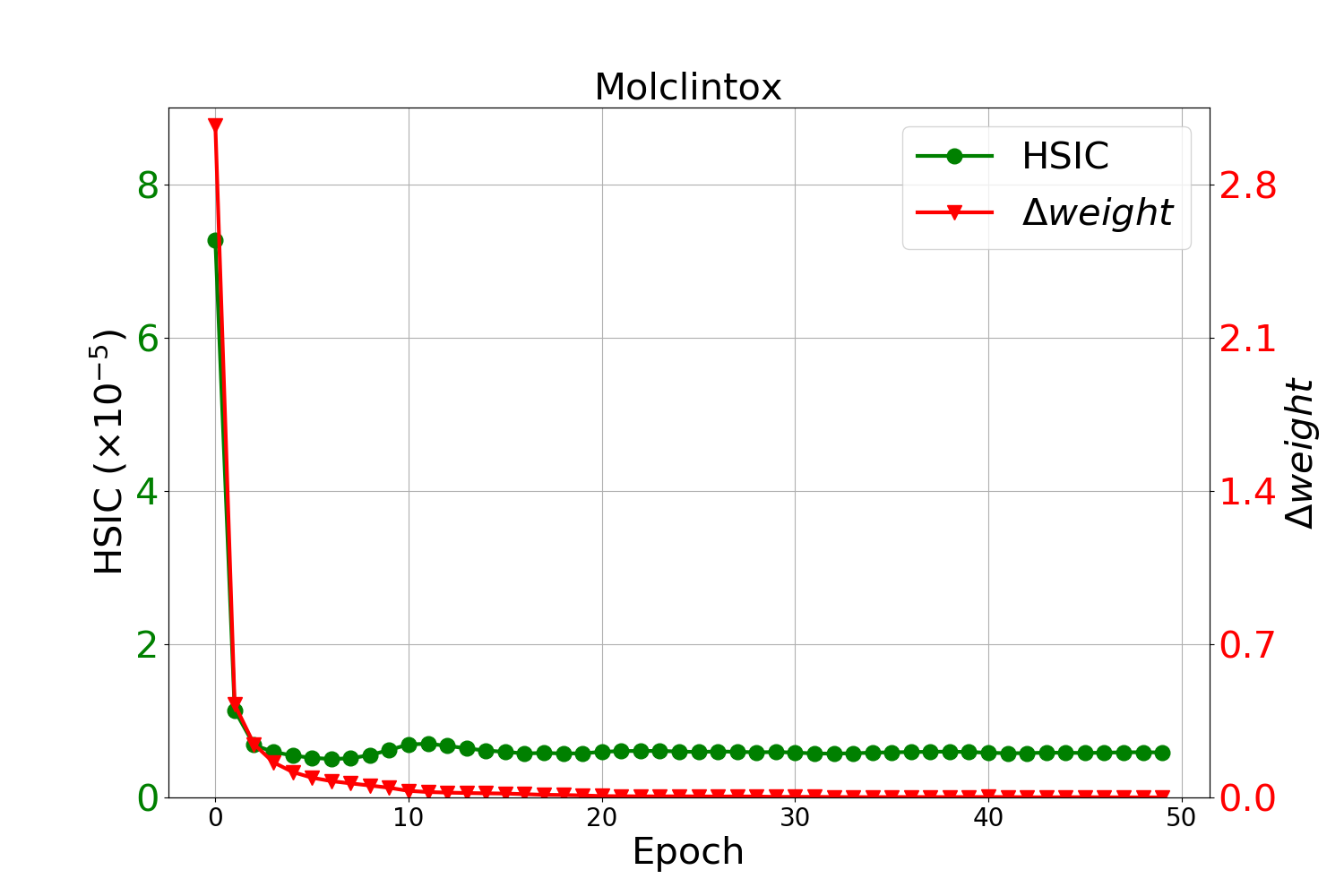}
\label{fig:gradient}
}
\caption{Convergence rate analysis of CVD regularizer.}
\label{fig:Convergence}
\end{figure*}


\paratitle{Comparison with Different Decorrelation Methods.} As there may be other ways to decorrelate the variables in GNNs by sample reweighting methods, one question arises naturally: is our proposed decorrelation framework a more suitable strategy for graphs? To answer this question, we compare the following alternatives:
\begin{itemize}
    \item GraphSAGE-Decorr: This method directly decorrelates each dimension of graph-level representations learned by GraphSAGE.
    \item StableGNN-NoCVD-Decorr: This method decorrelates each dimension of concatenated high-level representation $\mathbf{H}$ learned by StableGNN-NoCVD.
    \item StableGNN-NoCVD-Distent: This method forces the high-level representation learned by StableGNN-NoCVD to be disentangled by adding a HSIC regularizer to the overall loss.
\end{itemize}
The results are shown in Table~\ref{Tab::Difference decorrelation}. Compared with these potential decorrelation methods, StableGNN achieves better results consistently, demonstrating that decorrelating the representations by the cluster-level granularity is a more suitable strategy for graph data. Moreover, we find that GraphSAGE-Decorr/StableGNN-NoCVD-Decorr shows worse performance than GraphSAGE/StableGNN-NoCVD, and the reason is that if we aggressively decorrelate single dimension of embeddings, it will inevitably break the intrinsic semantic meaning of original data. Furthermore, StableGNN-NoCVD-Distent forces the high-level representations to be disentangled, which changes the semantic implication of features, while
StableGNN learns sample weights to adjust the data structure
while the semantics of features are not affected. Overall, StableGNN is a general and effective framework compared with all the potential ways.

\paratitle{Convergence Rate and Parameter Sensitivity Analysis.} We first analyze the convergence rate of our proposed causal variable distinguishing regularizer. We report the summation of HSIC value of all high-level representation pairs and the difference of learned weights between two consecutive epochs during the weights learning procedure in one batch on three relatively smaller datasets in Figure~\ref{fig:Convergence}. As we can see, the weights learned by CVD regularizer could achieve convergence very fast while reducing the HSIC value significantly. In addition, we study the sensitivity of the number of high-level representations and report the performance of StableGNN-NoCVD and StableGNN based on the same pre-defined number of clusters in Figure~\ref{fig:pre-defined maximum}. StableGNN outperforms StableGNN-NoCVD on almost all cases, which well demonstrates the robustness of our methods with the number of pre-defined clusters and the effectiveness of the proposed CVD regularizer. Note that our framework could learn the appropriate number of clusters in an end-to-end way, i.e., some clusters might not be used by the assignment matrix.

\begin{figure*}[!htbp]
\centering
\subfigure[Molbace($\uparrow$)]{
\includegraphics[ width=5.5cm]{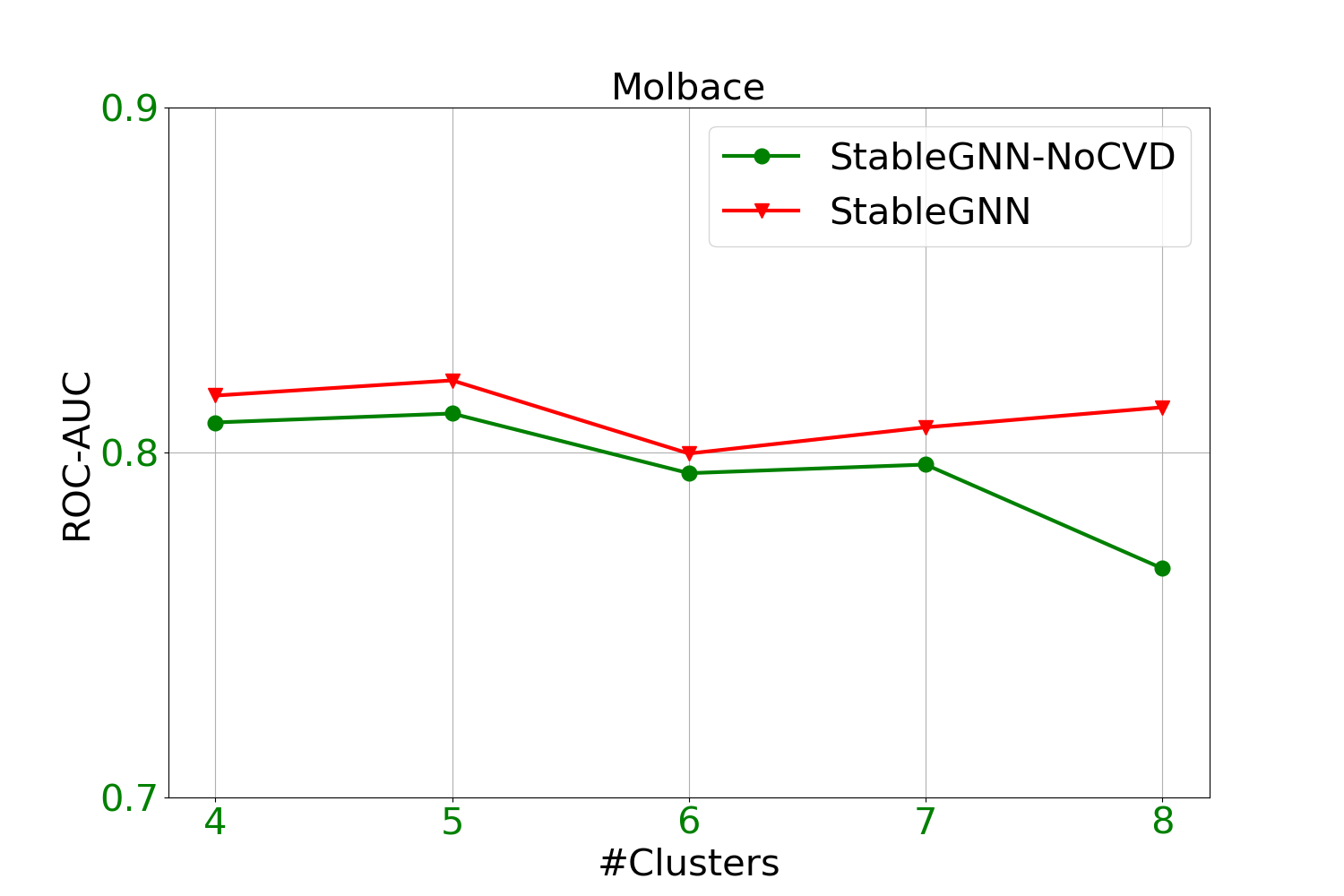}
\label{fig:gradient}
}
\subfigure[Molbbbp($\uparrow$)]{
\includegraphics[ width=5.5cm]{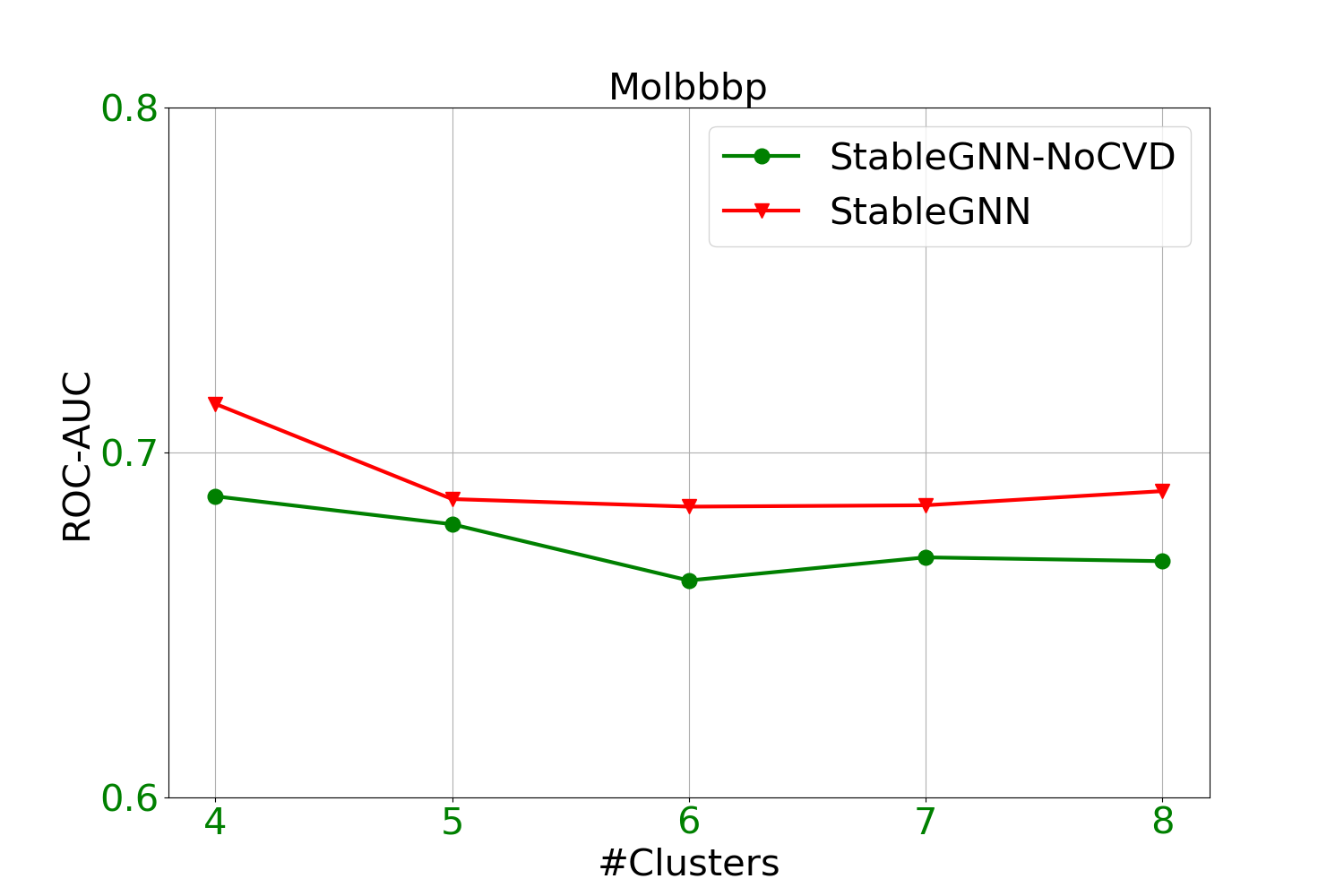}
\label{fig:gradient}
}
\subfigure[Molclintox($\uparrow$)]{
\includegraphics[width=5.5cm]{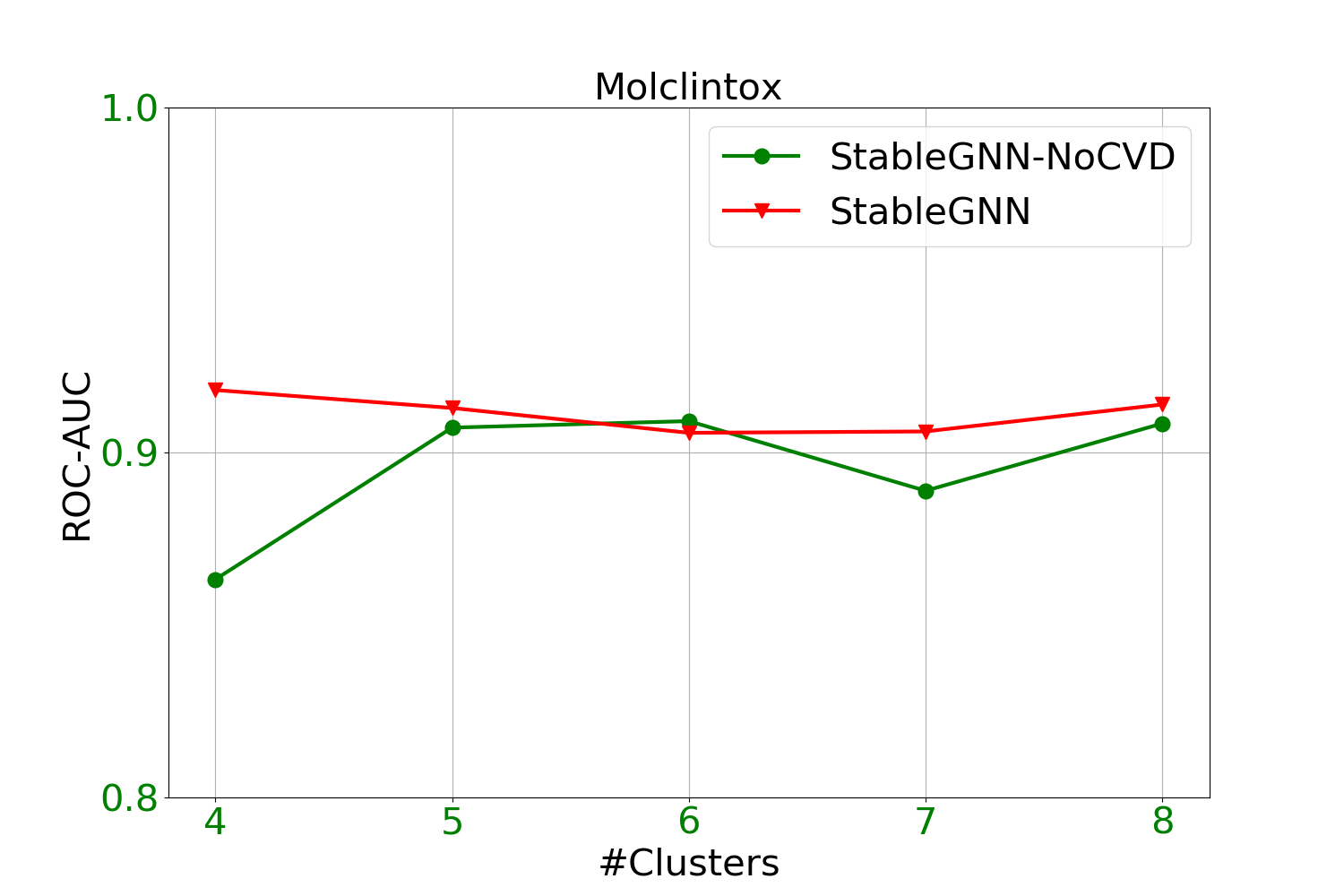}
\label{fig:gradient}
}
\caption{Sensitivity of the pre-defined maximum number of clusters.}
\label{fig:pre-defined maximum}
\end{figure*}
\begin{figure*}[!htbp]
	\centering
	\includegraphics[width=18cm]{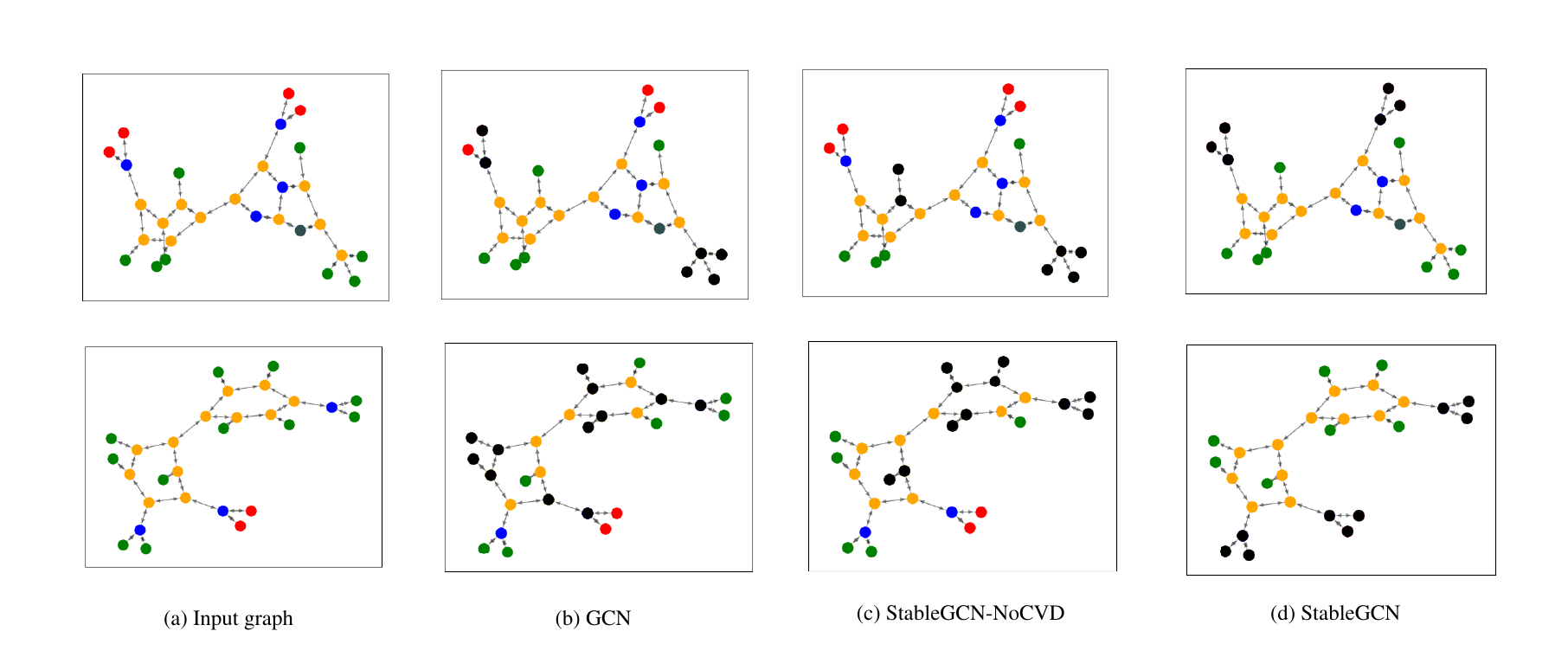}
	\caption{Explanation instance for MUTAG dataset. The blue, green, red and yellow colors represent N, H, O, C atoms, respectively. The top important subgraphs selected by GNNExplainer are viewed in black color (keep the top 6 important edges). The picture is best viewed in color.}
	\label{fig:MUTAG explainer}
\end{figure*}
\paratitle{Interpretability Analysis.} In Figure~\ref{fig:MUTAG explainer}, we investigate explanations for graph classification tasks on MUTAG dataset. In the MUTAG example, colors indicate node features, which represent atoms. StableGNN correctly identifies chemical $\text{NO}_2$ and $\text{NH}_2$, which are known to be mutagenic~\cite{luo2020parameterized} while baselines fail in. These cases demonstrate that our model could utilize more interpretable structures for prediction. Moreover, the difference of StableGCN-NoCVD between StableGCN indicates the necessity of  integrating two components in our framework.

\section{Conclusion and Future work}
In this paper, we study a practical but seldom studied problem: generalizing GNNs on out-of-distribution graph data. We analyze the problem in a causal view that the generalization of GNNs will be hindered by the spurious correlation among subgraphs. To improve the stability of existing methods, we propose a general causal representation learning framework, called StableGNN, which integrates graph high-level variable representation learning and causal effect estimation in a unified framework. Extensive experiments well demonstrate the effectiveness, flexibility, and interpretability of the StableGNN.
\par In addition, we believe that this paper just opens a direction for causal representation learning on graphs. As the most important contribution, we propose a general framework for causal graph representation learning: graph high-level variable representation learning and causal variable distinguishing, which can be flexibly adjusted for specific tasks. For example, besides molecules, we could adjust our framework to learn the causal substructures of proteins. The substructures of proteins play a crucial causal role in determining their properties and functions. Different substructures can confer distinct functionalities and characteristics on the proteins. To have a broader impact, we believe the idea could also spur causal representation learning in other areas, like object recognition~\cite{riesenhuber2000models}, multi-modal data fusion~\cite{wang2021survey}, and automatic driving in wild environments~\cite{bojarski2016end}.

\section*{Acknowledgments}
This work was supported in part by the National Natural Science
Foundation of China (No. U20B2045, 62192784, 62322203, 62172052, 62002029, 62141607, U1936219). This work was also supported in part by National Key R\&D Program of China (No. 2018AAA0102004).


%





\ifCLASSOPTIONcaptionsoff
  \newpage
\fi
\bibliographystyle{IEEEtran}
\bibliography{ref}

\begin{thebibliography}{10}
\providecommand{\url}[1]{#1}
\csname url@samestyle\endcsname
\providecommand{\newblock}{\relax}
\providecommand{\bibinfo}[2]{#2}
\providecommand{\BIBentrySTDinterwordspacing}{\spaceskip=0pt\relax}
\providecommand{\BIBentryALTinterwordstretchfactor}{4}
\providecommand{\BIBentryALTinterwordspacing}{\spaceskip=\fontdimen2\font plus
\BIBentryALTinterwordstretchfactor\fontdimen3\font minus \fontdimen4\font\relax}
\providecommand{\BIBforeignlanguage}[2]{{%
\expandafter\ifx\csname l@#1\endcsname\relax
\typeout{** WARNING: IEEEtran.bst: No hyphenation pattern has been}%
\typeout{** loaded for the language `#1'. Using the pattern for}%
\typeout{** the default language instead.}%
\else
\language=\csname l@#1\endcsname
\fi
#2}}
\providecommand{\BIBdecl}{\relax}
\BIBdecl

\bibitem{ying2019gnnexplainer}
R.~Ying, D.~Bourgeois, J.~You, M.~Zitnik, and J.~Leskovec, ``Gnnexplainer: Generating explanations for graph neural networks,'' in \emph{NeurIPS}, 2019.

\bibitem{scarselli2008graph}
F.~Scarselli, M.~Gori, A.~C. Tsoi, M.~Hagenbuchner, and G.~Monfardini, ``The graph neural network model,'' \emph{IEEE Transactions on Neural Networks}, vol.~20, no.~1, pp. 61--80, 2008.

\bibitem{kipf2016semi}
T.~N. Kipf and M.~Welling, ``Semi-supervised classification with graph convolutional networks,'' in \emph{ICLR}, 2016.

\bibitem{velivckovic2017graph}
P.~Veli{\v{c}}kovi{\'c}, G.~Cucurull, A.~Casanova, A.~Romero, P.~Lio, and Y.~Bengio, ``Graph attention networks,'' in \emph{ICLR}, 2017.

\bibitem{hamilton2017inductive}
W.~Hamilton, Z.~Ying, and J.~Leskovec, ``Inductive representation learning on large graphs,'' in \emph{NeurIPS}, 2017, pp. 1024--1034.

\bibitem{hu2020open}
W.~Hu, M.~Fey, M.~Zitnik, Y.~Dong, H.~Ren, B.~Liu, M.~Catasta, and J.~Leskovec, ``Open graph benchmark: Datasets for machine learning on graphs,'' in \emph{NeurIPS}, 2020.

\bibitem{lee2018graph}
J.~B. Lee, R.~Rossi, and X.~Kong, ``Graph classification using structural attention,'' in \emph{SIGKDD}, 2018, pp. 1666--1674.

\bibitem{ying2018hierarchical}
R.~Ying, J.~You, C.~Morris, X.~Ren, W.~L. Hamilton, and J.~Leskovec, ``Hierarchical graph representation learning with differentiable pooling,'' in \emph{NeurIPS}, 2018.

\bibitem{pope2019explainability}
P.~E. Pope, S.~Kolouri, M.~Rostami, C.~E. Martin, and H.~Hoffmann, ``Explainability methods for graph convolutional neural networks,'' in \emph{CVPR}, 2019, pp. 10\,772--10\,781.

\bibitem{zhang2018end}
M.~Zhang, Z.~Cui, M.~Neumann, and Y.~Chen, ``An end-to-end deep learning architecture for graph classification,'' in \emph{AAAI}, 2018.

\bibitem{liao2020pac}
R.~Liao, R.~Urtasun, and R.~Zemel, ``A pac-bayesian approach to generalization bounds for graph neural networks,'' in \emph{ICLR}, 2020.

\bibitem{bengio2019meta}
Y.~Bengio, T.~Deleu, N.~Rahaman, R.~Ke, S.~Lachapelle, O.~Bilaniuk, A.~Goyal, and C.~Pal, ``A meta-transfer objective for learning to disentangle causal mechanisms,'' \emph{arXiv preprint arXiv:1901.10912}, 2019.

\bibitem{engstrom2019exploring}
L.~Engstrom, B.~Tran, D.~Tsipras, L.~Schmidt, and A.~Madry, ``Exploring the landscape of spatial robustness,'' in \emph{ICML}.\hskip 1em plus 0.5em minus 0.4em\relax PMLR, 2019, pp. 1802--1811.

\bibitem{su2019one}
J.~Su, D.~V. Vargas, and K.~Sakurai, ``One pixel attack for fooling deep neural networks,'' \emph{TEC}, vol.~23, no.~5, pp. 828--841, 2019.

\bibitem{hendrycks2019benchmarking}
D.~Hendrycks and T.~Dietterich, ``Benchmarking neural network robustness to common corruptions and perturbations,'' \emph{arXiv preprint arXiv:1903.12261}, 2019.

\bibitem{sun2019test}
Y.~Sun, X.~Wang, Z.~Liu, J.~Miller, A.~Efros, and M.~Hardt, ``Test-time training with self-supervision for generalization under distribution shifts,'' in \emph{Proc. Int. Conf. Mach. Learn.}\hskip 1em plus 0.5em minus 0.4em\relax PMLR, 2020, pp. 9229--9248.

\bibitem{krueger2021out}
D.~Krueger, E.~Caballero, J.-H. Jacobsen, A.~Zhang, J.~Binas, D.~Zhang, R.~Le~Priol, and A.~Courville, ``Out-of-distribution generalization via risk extrapolation (rex),'' in \emph{ICML}, 2021, pp. 5815--5826.

\bibitem{zhang2021deep}
X.~Zhang, P.~Cui, R.~Xu, L.~Zhou, Y.~He, and Z.~Shen, ``Deep stable learning for out-of-distribution generalization,'' in \emph{CVPR}, 2021, pp. 5372--5382.

\bibitem{lake2017building}
B.~M. Lake, T.~D. Ullman, J.~B. Tenenbaum, and S.~J. Gershman, ``Building machines that learn and think like people,'' \emph{Behavioral and brain sciences}, vol.~40, 2017.

\bibitem{lopez2017discovering}
D.~Lopez-Paz, R.~Nishihara, S.~Chintala, B.~Scholkopf, and L.~Bottou, ``Discovering causal signals in images,'' in \emph{CVPR}, 2017, pp. 6979--6987.

\bibitem{jin2020hierarchical}
W.~Jin, R.~Barzilay, and T.~Jaakkola, ``Hierarchical generation of molecular graphs using structural motifs,'' in \emph{ICML}.\hskip 1em plus 0.5em minus 0.4em\relax PMLR, 2020, pp. 4839--4848.

\bibitem{bo2021beyond}
D.~Bo, X.~Wang, C.~Shi, and H.~Shen, ``Beyond low-frequency information in graph convolutional networks,'' in \emph{AAAI}, 2021.

\bibitem{chen2023universal}
J.~Chen, F.~Dai, X.~Gu, J.~Zhou, B.~Li, and W.~Wang, ``Universal domain adaptive network embedding for node classification,'' in \emph{Proceedings of the 31st ACM International Conference on Multimedia}, 2023, pp. 4022--4030.

\bibitem{schlichtkrull2018modeling}
M.~Schlichtkrull, T.~N. Kipf, P.~Bloem, R.~Van Den~Berg, I.~Titov, and M.~Welling, ``Modeling relational data with graph convolutional networks,'' in \emph{European semantic web conference}.\hskip 1em plus 0.5em minus 0.4em\relax Springer, 2018, pp. 593--607.

\bibitem{zhang2018link}
M.~Zhang and Y.~Chen, ``Link prediction based on graph neural networks,'' in \emph{NeurIPS}, vol.~31, 2018, pp. 5165--5175.

\bibitem{fan2019metapath}
S.~Fan, J.~Zhu, X.~Han, C.~Shi, L.~Hu, B.~Ma, and Y.~Li, ``Metapath-guided heterogeneous graph neural network for intent recommendation,'' in \emph{SIGKDD}, 2019, pp. 2478--2486.

\bibitem{pan2023beyond}
E.~Pan and Z.~Kang, ``Beyond homophily: Reconstructing structure for graph-agnostic clustering,'' in \emph{ICML}, 2023.

\bibitem{fan2020one2multi}
S.~Fan, X.~Wang, C.~Shi, E.~Lu, K.~Lin, and B.~Wang, ``One2multi graph autoencoder for multi-view graph clustering,'' in \emph{WWW}, 2020, pp. 3070--3076.

\bibitem{wang2017community}
X.~Wang, P.~Cui, J.~Wang, J.~Pei, W.~Zhu, and S.~Yang, ``Community preserving network embedding,'' in \emph{Proceedings of the AAAI conference on artificial intelligence}, vol.~31, no.~1, 2017.

\bibitem{lee2019self}
J.~Lee, I.~Lee, and J.~Kang, ``Self-attention graph pooling,'' in \emph{ICML}.\hskip 1em plus 0.5em minus 0.4em\relax PMLR, 2019, pp. 3734--3743.

\bibitem{gao2019graph}
H.~Gao and S.~Ji, ``Graph u-nets,'' in \emph{ICML}.\hskip 1em plus 0.5em minus 0.4em\relax PMLR, 2019, pp. 2083--2092.

\bibitem{li2022ood}
H.~Li, X.~Wang, Z.~Zhang, and W.~Zhu, ``Ood-gnn: Out-of-distribution generalized graph neural network,'' \emph{IEEE Transactions on Knowledge and Data Engineering}, 2022.

\bibitem{wu2022discovering}
Y.-X. Wu, X.~Wang, A.~Zhang, X.~He, and T.-S. Chua, ``Discovering invariant rationales for graph neural networks,'' in \emph{ICLR}, 2022.

\bibitem{chen2022learning}
Y.~Chen, Y.~Zhang, Y.~Bian, H.~Yang, M.~Kaili, B.~Xie, T.~Liu, B.~Han, and J.~Cheng, ``Learning causally invariant representations for out-of-distribution generalization on graphs,'' in \emph{NeurIPS}, vol.~35, 2022, pp. 22\,131--22\,148.

\bibitem{li2022learning}
H.~Li, Z.~Zhang, X.~Wang, and W.~Zhu, ``Learning invariant graph representations for out-of-distribution generalization,'' in \emph{Advances in Neural Information Processing Systems}, 2022.

\bibitem{fan2022debiasing}
S.~Fan, X.~Wang, Y.~Mo, C.~Shi, and J.~Tang, ``Debiasing graph neural networks via learning disentangled causal substructure,'' \emph{NeurIPS}, vol.~35, pp. 24\,934--24\,946, 2022.

\bibitem{yang2022learning}
N.~Yang, K.~Zeng, Q.~Wu, X.~Jia, and J.~Yan, ``Learning substructure invariance for out-of-distribution molecular representations,'' in \emph{NeurIPS}, 2022.

\bibitem{scholkopf2021toward}
B.~Sch{\"o}lkopf, F.~Locatello, S.~Bauer, N.~R. Ke, N.~Kalchbrenner, A.~Goyal, and Y.~Bengio, ``Toward causal representation learning,'' \emph{Proceedings of the IEEE}, vol. 109, no.~5, pp. 612--634, 2021.

\bibitem{arjovsky2019invariant}
M.~Arjovsky, L.~Bottou, I.~Gulrajani, and D.~Lopez-Paz, ``Invariant risk minimization,'' \emph{arXiv preprint arXiv:1907.02893}, 2019.

\bibitem{rosenfeld2020risks}
E.~Rosenfeld, P.~Ravikumar, and A.~Risteski, ``The risks of invariant risk minimization,'' in \emph{ICLR}, 2020.

\bibitem{kamath2021does}
P.~Kamath, A.~Tangella, D.~Sutherland, and N.~Srebro, ``Does invariant risk minimization capture invariance?'' in \emph{International Conference on Artificial Intelligence and Statistics}.\hskip 1em plus 0.5em minus 0.4em\relax PMLR, 2021, pp. 4069--4077.

\bibitem{qiao2020learning}
F.~Qiao, L.~Zhao, and X.~Peng, ``Learning to learn single domain generalization,'' in \emph{CVPR}, 2020, pp. 12\,556--12\,565.

\bibitem{matsuura2020domain}
T.~Matsuura and T.~Harada, ``Domain generalization using a mixture of multiple latent domains,'' in \emph{AAAI}, vol.~34, no.~07, 2020, pp. 11\,749--11\,756.

\bibitem{wang2019learning}
H.~Wang, Z.~He, Z.~C. Lipton, and E.~P. Xing, ``Learning robust representations by projecting superficial statistics out,'' \emph{arXiv preprint arXiv:1903.06256}, 2019.

\bibitem{kuang2018stable}
K.~Kuang, P.~Cui, S.~Athey, R.~Xiong, and B.~Li, ``Stable prediction across unknown environments,'' in \emph{SIGKDD}, 2018, pp. 1617--1626.

\bibitem{kuang2020stable}
K.~Kuang, R.~Xiong, P.~Cui, S.~Athey, and B.~Li, ``Stable prediction with model misspecification and agnostic distribution shift,'' in \emph{Proceedings of the AAAI Conference on Artificial Intelligence}, vol.~34, no.~04, 2020, pp. 4485--4492.

\bibitem{angrist1995identification}
G.~W. Imbens and J.~D. Angrist, ``Identification and estimation of local average treatment effects,'' \emph{Econometrica}, pp. 467--475, 1994.

\bibitem{shen2018causally}
Z.~Shen, P.~Cui, K.~Kuang, B.~Li, and P.~Chen, ``Causally regularized learning with agnostic data selection bias,'' in \emph{ACM MM}, 2018, pp. 411--419.

\bibitem{shen2020sample}
Z.~Shen, P.~Cui, T.~Zhang, and K.~Kuang, ``Stable learning via sample reweighting.'' in \emph{AAAI}, 2020, pp. 5692--5699.

\bibitem{luo2020parameterized}
D.~Luo, W.~Cheng, D.~Xu, W.~Yu, B.~Zong, H.~Chen, and X.~Zhang, ``Parameterized explainer for graph neural network,'' in \emph{NeurIPS}, 2020.

\bibitem{keriven2019universal}
N.~Keriven and G.~Peyr{\'e}, ``Universal invariant and equivariant graph neural networks,'' \emph{NeurIPS}, vol.~32, pp. 7092--7101, 2019.

\bibitem{leman1968reduction}
A.~Leman and B.~Weisfeiler, ``A reduction of a graph to a canonical form and an algebra arising during this reduction,'' \emph{Nauchno-Technicheskaya Informatsiya}, vol.~2, no.~9, pp. 12--16, 1968.

\bibitem{hainmueller2012entropy}
J.~Hainmueller, ``Entropy balancing for causal effects: A multivariate reweighting method to produce balanced samples in observational studies,'' \emph{Political Analysis}, vol.~20, no.~1, pp. 25--46, 2012.

\bibitem{zubizarreta2015stable}
J.~R. Zubizarreta, ``Stable weights that balance covariates for estimation with incomplete outcome data,'' \emph{Journal of the American Statistical Association}, vol. 110, no. 511, pp. 910--922, 2015.

\bibitem{athey2016approximate}
S.~Athey, G.~W. Imbens, and S.~Wager, ``Approximate residual balancing: De-biased inference of average treatment effects in high dimensions,'' \emph{arXiv preprint arXiv:1604.07125}, 2016.

\bibitem{song2007supervised}
L.~Song, A.~Smola, A.~Gretton, K.~M. Borgwardt, and J.~Bedo, ``Supervised feature selection via dependence estimation,'' in \emph{ICML}, 2007, pp. 823--830.

\bibitem{gretton2005measuring}
A.~Gretton, O.~Bousquet, A.~Smola, and B.~Sch{\"o}lkopf, ``Measuring statistical dependence with hilbert-schmidt norms,'' in \emph{International conference on algorithmic learning theory}.\hskip 1em plus 0.5em minus 0.4em\relax Springer, 2005, pp. 63--77.

\bibitem{zou2020counterfactual}
H.~Zou, P.~Cui, B.~Li, Z.~Shen, J.~Ma, H.~Yang, and Y.~He, ``Counterfactual prediction for bundle treatment,'' \emph{NeurIPS}, vol.~33, 2020.

\bibitem{kingma2014adam}
D.~P. Kingma and J.~Ba, ``Adam: A method for stochastic optimization,'' \emph{arXiv preprint arXiv:1412.6980}, 2014.

\bibitem{song2012feature}
L.~Song, A.~Smola, A.~Gretton, J.~Bedo, and K.~Borgwardt, ``Feature selection via dependence maximization.'' \emph{Journal of Machine Learning Research}, vol.~13, no.~5, 2012.

\bibitem{lin2021generative}
W.~Lin, H.~Lan, and B.~Li, ``Generative causal explanations for graph neural networks,'' in \emph{ICML}, 2021.

\bibitem{ioffe1502accelerating}
S.~Ioffe and C.~S.~B. Normalization, ``Accelerating deep network training by reducing internal covariate shift,'' \emph{arXiv preprint arXiv:1502.03167}.

\bibitem{he2016deep}
K.~He, X.~Zhang, S.~Ren, and J.~Sun, ``Deep residual learning for image recognition,'' in \emph{CVPR}, 2016, pp. 770--778.

\bibitem{huang2005using}
J.~Huang and C.~X. Ling, ``Using auc and accuracy in evaluating learning algorithms,'' \emph{TKDE}, vol.~17, no.~3, pp. 299--310, 2005.

\bibitem{landrum2006rdkit}
\BIBentryALTinterwordspacing
{RDKit: Open-source cheminformatics}. [Online]. Available: \url{{https://www.rdkit.org}}
\BIBentrySTDinterwordspacing

\bibitem{wu2018moleculenet}
Z.~Wu, B.~Ramsundar, E.~N. Feinberg, J.~Gomes, C.~Geniesse, A.~S. Pappu, K.~Leswing, and V.~Pande, ``Moleculenet: a benchmark for molecular machine learning,'' \emph{Chemical science}, vol.~9, no.~2, pp. 513--530, 2018.

\bibitem{ishida2021graph}
S.~Ishida, T.~Miyazaki, Y.~Sugaya, and S.~Omachi, ``Graph neural networks with multiple feature extraction paths for chemical property estimation,'' \emph{Molecules}, vol.~26, no.~11, p. 3125, 2021.

\bibitem{debnath1991structure}
A.~K. Debnath, R.~L. Lopez~de Compadre, G.~Debnath, A.~J. Shusterman, and C.~Hansch, ``Structure-activity relationship of mutagenic aromatic and heteroaromatic nitro compounds. correlation with molecular orbital energies and hydrophobicity,'' \emph{Journal of medicinal chemistry}, vol.~34, no.~2, pp. 786--797, 1991.

\bibitem{landrum2013rdkit}
\BIBentryALTinterwordspacing
{RDKit: A software suite for cheminformatics, computational chemistry, and predictive modeling}. [Online]. Available: \url{{http://www.rdkit.org/RDKit_Overview.pdf}}
\BIBentrySTDinterwordspacing

\bibitem{shervashidze2011weisfeiler}
N.~Shervashidze, P.~Schweitzer, E.~J. Van~Leeuwen, K.~Mehlhorn, and K.~M. Borgwardt, ``Weisfeiler-lehman graph kernels.'' \emph{Journal of Machine Learning Research}, vol.~12, no.~9, 2011.

\bibitem{shervashidze2009efficient}
N.~Shervashidze, S.~Vishwanathan, T.~Petri, K.~Mehlhorn, and K.~Borgwardt, ``Efficient graphlet kernels for large graph comparison,'' in \emph{Artificial intelligence and statistics}.\hskip 1em plus 0.5em minus 0.4em\relax PMLR, 2009, pp. 488--495.

\bibitem{xu2018how}
\BIBentryALTinterwordspacing
K.~Xu, W.~Hu, J.~Leskovec, and S.~Jegelka, ``How powerful are graph neural networks?'' in \emph{ICLR}, 2019. [Online]. Available: \url{https://openreview.net/forum?id=ryGs6iA5Km}
\BIBentrySTDinterwordspacing

\bibitem{monti2017geometric}
F.~Monti, D.~Boscaini, J.~Masci, E.~Rodola, J.~Svoboda, and M.~M. Bronstein, ``Geometric deep learning on graphs and manifolds using mixture model cnns,'' in \emph{CVPR}, 2017, pp. 5115--5124.

\bibitem{pmlr-v97-wu19e}
F.~Wu, A.~Souza, T.~Zhang, C.~Fifty, T.~Yu, and K.~Weinberger, ``Simplifying graph convolutional networks,'' in \emph{ICML}.\hskip 1em plus 0.5em minus 0.4em\relax PMLR, 2019, pp. 6861--6871.

\bibitem{xu2018representation}
K.~Xu, C.~Li, Y.~Tian, T.~Sonobe, K.-i. Kawarabayashi, and S.~Jegelka, ``Representation learning on graphs with jumping knowledge networks,'' in \emph{ICML}.\hskip 1em plus 0.5em minus 0.4em\relax PMLR, 2018, pp. 5453--5462.

\bibitem{riesenhuber2000models}
M.~Riesenhuber and T.~Poggio, ``Models of object recognition,'' \emph{Nature neuroscience}, vol.~3, no.~11, pp. 1199--1204, 2000.

\bibitem{wang2021survey}
Y.~Wang, ``Survey on deep multi-modal data analytics: Collaboration, rivalry, and fusion,'' \emph{ACM Transactions on Multimedia Computing, Communications, and Applications (TOMM)}, vol.~17, no.~1s, pp. 1--25, 2021.

\bibitem{bojarski2016end}
M.~Bojarski, D.~Del~Testa, D.~Dworakowski, B.~Firner, B.~Flepp, P.~Goyal, L.~D. Jackel, M.~Monfort, U.~Muller, J.~Zhang \emph{et~al.}, ``End to end learning for self-driving cars,'' \emph{arXiv preprint arXiv:1604.07316}, 2016.

\end{thebibliography}



%

%

\begin{IEEEbiography}[{\includegraphics[width=1in,height=1.25in,clip,keepaspectratio]{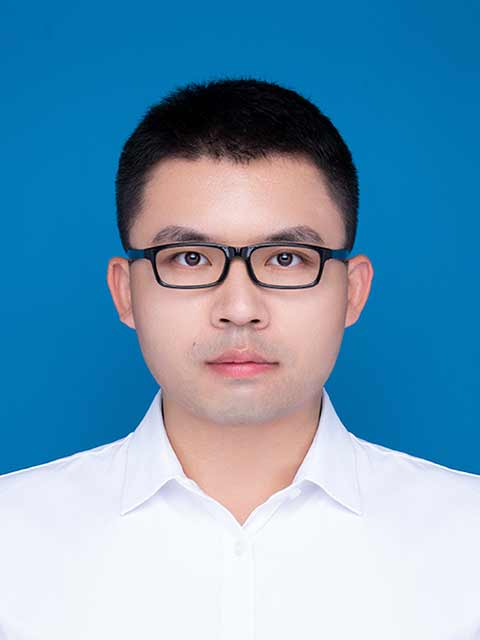}}]{Shaohua Fan}
received the Ph.D. degree in 2022 from Beijing University of Posts and Telecommunications. He was a visiting student at Mila for one year. Currently, he is a Postdoc researcher in the Department of Computer
Science of Tsinghua University.
His main research interests including graph mining, causal machine learning, and causal discovery. He has published several papers in major international conferences and journals, including NeurIPS, KDD, WWW and TNNLS etc.

\end{IEEEbiography}

\begin{IEEEbiography}[{\includegraphics[width=1in,height=1.25in,clip,keepaspectratio]{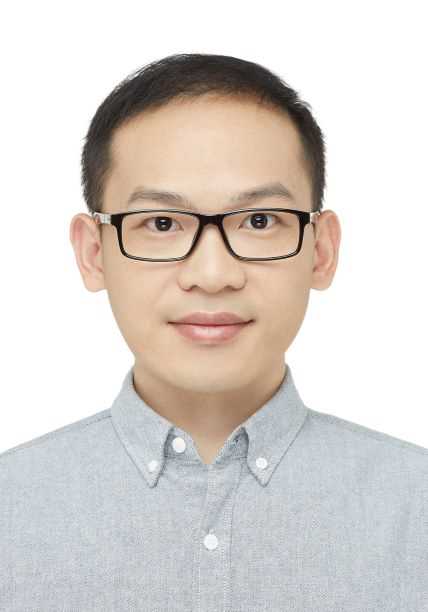}}]{Xiao Wang}
is an Associate Professor at Beihang University. He was an Associate Professor at Beijing University of Posts and Telecommunications. He received his Ph.D. degree from the School of Computer Science and Technology, Tianjin University, Tianjin, China, in 2016. He was a postdoctoral researcher in Tsinghua University. His current research interests include data mining, social network analysis, and machine learning. Until now, he has published
more than 70 papers in conferences such as NeurIPS, AAAI, IJCAI, WWW, KDD, etc. and journals such as IEEE TKDE, IEEE Trans. on Cybernetics, etc.
\end{IEEEbiography}


\begin{IEEEbiography}[{\includegraphics[width=1in,height=1.25in,clip,keepaspectratio]{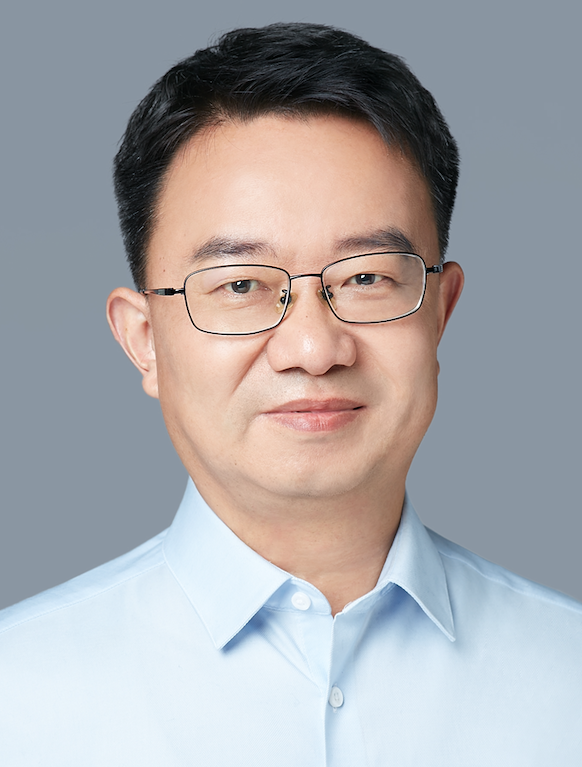}}]{Chuan Shi}
received the B.S. degree from the Jilin University in 2001, the M.S. degree from
the Wuhan University in 2004, and Ph.D. degree from the ICT of Chinese Academic of Sciences
in 2007. He joined the Beijing University of Posts and Telecommunications as a lecturer in 2007,
and is a professor and deputy director of Beijing Key Lab of Intelligent Telecommunications Software and Multimedia at present. His research interests are in data mining, machine learning, and evolutionary computing. He has published
more than 100 papers in refereed journals and conferences, such as TKDE, KDD, WWW, NeurIPS, and ICLR.
\end{IEEEbiography}

\begin{IEEEbiography}[{\includegraphics[width=1in,height=1.35in,clip,keepaspectratio]{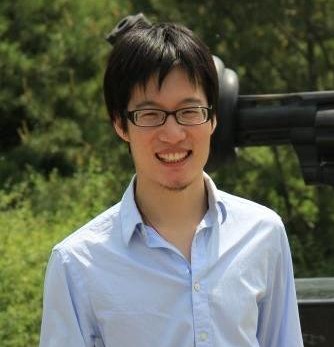}}]{Peng Cui} is an Associate Professor with tenure in Tsinghua University. He got his PhD degree from Tsinghua University in 2010. His research interests include causally-regularized machine learning, network representation learning, and social dynamics modeling. He has published more than 100 papers in prestigious conferences and journals in data mining and multimedia.  He received ACM China Rising Star Award in 2015, and CCF-IEEE CS Young Scientist Award in 2018. He is now a Distinguished Member of ACM and CCF, and a Senior Member of IEEE.
\end{IEEEbiography}
\begin{IEEEbiography}[{\includegraphics[width=1in,height=1.25in,clip,keepaspectratio]{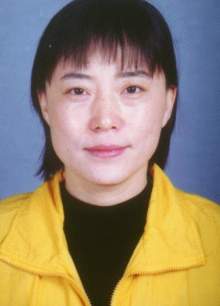}}]{Bai Wang}
received the B.S. degree from the Xian Jiaotong University, Xian, China and Ph.D. degree from the Beijing University of Posts and Telecommunications, Beijing, China. And she is currently a professor of computer science in BUPT. She was the director of Beijing Key Lab of Intelligent Telecommunications Software and Multimedia.
\end{IEEEbiography}




\end{document}